\documentclass[10pt,a4paper]{article}
\usepackage[textwidth=171mm,textheight=256mm,centering]{geometry}
\usepackage{graphicx,tikz,amsmath,amssymb,amsfonts,bm,booktabs}
\usetikzlibrary{arrows.meta,calc,shapes.geometric}
\definecolor{purple}{HTML}{6633CC}
\definecolor{teal}{HTML}{008E99}
\definecolor{goal}{HTML}{C61D66}
\usepackage{url,capt-of,ragged2e,longtable,multirow,float,placeins,needspace}
\usepackage{algorithm,algorithmic,amsthm,enumitem}
\usepackage{natbib}
\setcitestyle{authoryear,round,citesep={;},aysep={,},yysep={;}}
\usepackage{hyperref}
\usepackage{microtype}
\usepackage[font=small,labelfont=bf]{caption}
\newtheorem{assumption}{Assumption}
\newtheorem{proposition}{Proposition}
\newtheorem{lemma}{Lemma}

\theoremstyle{definition}

\theoremstyle{remark}

\ifdefined\pdfminorversion\pdfminorversion=7\fi

\title{Planning under Noise with Neural Operators}
\author{Michael Chen\textsuperscript{1,*}, Hong Chul Nam\textsuperscript{2,*}, Xi Deng\textsuperscript{1,*}, Anima Anandkumar\textsuperscript{1}\\[0.5em]
{\small \textsuperscript{1}Computing and Mathematical Sciences, California Institute of Technology}\\[0.25em]
{\small \textsuperscript{2}Center for Data Science, New York University}\\[0.25em]
{\small \textsuperscript{*}Equal contribution}}\date{}
\begin{document}

\maketitle

\begin{abstract}
Planning under uncertain control execution noise requires reliable feedback policies across diverse environments. We introduce \textbf{Planning Under Noise (PUN)}, a neural operator that learns noise-aware value fields for planning in unseen geometries. Building on an established screened-Poisson formulation of stochastic optimal control, we address two challenges: specifying boundary conditions that support navigation across varying geometries and efficiently computing fields with reliable gradients. Zero-Dirichlet obstacle boundaries terminate sampled paths at walls, weakening the goal signal along long routes and making it harder to estimate reliably. We instead impose reflecting Neumann conditions on obstacles and a Dirichlet condition at the goal, allowing sampled paths to continue exploring free space until they reach the target. To solve the resulting PDE without a configuration-space mesh, we extend a Monte Carlo (MC) solver, \textbf{Walk on Stars (WoSt)}, to higher dimensions. However, local MC fluctuations can reverse estimated gradient directions and cause planning failures, making direct planning require costly sampling. PUN amortizes this computation by learning geometry-to-field mappings from noisy, partially converged Monte Carlo labels, avoiding expensive convergence for every training example. At inference, test-time scaling uses bounded local search guided by the predicted global field. We evaluate PUN on planning tasks across dimensions. Across four noise levels, PUN matches or outperforms all compared planners, achieving 95--100\% success on the 4D replanning benchmark while maintaining the lowest runtime. At the highest tested noise level in 4D ($\lambda=4$), PUN achieves 10.7--60.6\% relative gains in success rate over other methods while running 3.29--32.15$\times$ faster.
\end{abstract}

\section{Introduction}

Planning~\citep{latombe1991robot} is one of the most fundamental problems in robotics. Given a robot's configuration space~\citep{lavalle2006planning}, planning computes a path from the current configuration to a goal that is optimal under a given cost function. Real robots do not execute a commanded motion exactly~\citep{thrun2002probabilistic}. These errors accumulate along a motion, so a route planned with a few centimetres of clearance can be executed into a wall. Sampling-based planners such as PRM~\citep{kavraki1996prm} and RRT~\citep{lavalle2001rrt}, and graph search such as A$^\star$~\citep{hart1968astar}, plan for a nominal system that executes commands exactly and leave the noise to a tracking controller. Sampling-based methods such as MPPI~\citep{williams2017mppi} and PI$^2$-CMA~\citep{stulp2012pi2cma} optimize control sequences or policy parameters using sampled trajectory costs. Optimal feedback under execution noise can be derived from a value function, which gives the minimum expected cost of reaching the goal from each configuration and satisfies the nonlinear Hamilton--Jacobi--Bellman (HJB) equation~\citep{fleming2006controlled}. Solving this equation on a grid scales exponentially with the dimension of the configuration space~\citep{bellman1957dynamic,han2018solving}.

While the problem can be converted into a linear screened Poisson equation~\citep{theodorou2010generalized} under control--noise matching (Assumption~\ref{ass:matching}), a central difficulty in Poisson-based planning is specifying boundary conditions that produce a reliable value field across varying geometries.

\begin{figure}[t]
\centering
\includegraphics[width=1.0\linewidth]{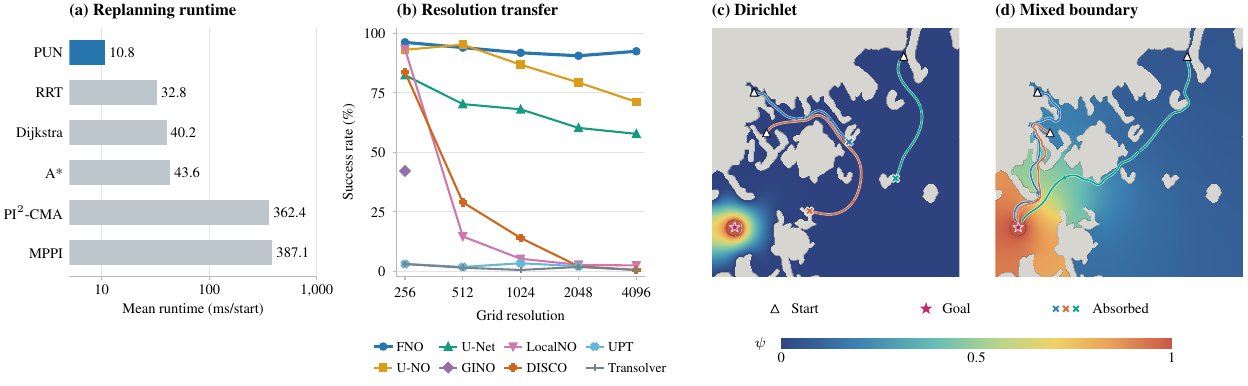}
\caption{\textbf{(a) Fast replanning.} PUN has the lowest mean runtime on unseen 4D maps: 10.8 ms/start, compared with 32.8--387.1 ms/start for the baselines. Runtimes are averaged equally over the four execution-noise levels in Table~\ref{tab:replanning} and plotted on a logarithmic axis. \textbf{(b) Resolution transfer.} FNO maintains over 90\% local-search success across all five resolutions without retraining, reaching 92.50\% at $4096\times4096$ after training at $256\times256$. Missing points indicate out-of-memory failures. \textbf{(c, d) Continued goal seeking.} In the illustrated geometry, the solution induced by the Dirichlet boundary cannot propagate goal information further since most trajectories are absorbed by obstacles. Consequently, it is more difficult for planners to find the target than with the mixed boundary condition.}
\label{fig:main-collage}
\end{figure}

Conventional mesh-based numerical methods can solve the screened Poisson equation efficiently in two and three dimensions. However, a robot with \(d\) degrees of freedom (DoF) induces a \(d\)-dimensional configuration space, making volumetric discretizations and numerical PDE solution increasingly expensive. Planning must also accommodate complex and previously unseen obstacle geometries, for which constructing a new mesh and solving the PDE from scratch can be costly. Mesh-free Monte Carlo methods~\citep{muller1956some,sawhney2020wos,sawhney2023wost,jambon2026walk} support PDE solution on complex geometries, motivating their use for planning in higher-dimensional configuration spaces.

Among these, Walk-on-Spheres (WoS)~\citep{muller1956some} estimates the solution at a point by jumping repeatedly to a uniformly sampled point on the largest ball fitting in free space, until it reaches a boundary of known value, which requires every reachable boundary to be Dirichlet. Walk-on-Stars (WoSt)~\citep{sawhney2023wost} accommodates reflecting obstacles by replacing that ball with a star-shaped region that may contain part of the obstacle boundary, keeping every boundary point visible from the center so that one ray continues the walk instead of branching at each contact. The screened kernels for Walk-on-Stars have been derived only in two and three dimensions~\citep{sawhney2023wost}. Moreover, directly planning with Monte Carlo estimates is unreliable: trajectories depend on local value gradients, and even modest estimator noise can corrupt or reverse the estimated direction, leading to planning failure.

Neural operators~\citep{kovachki2023neural} can amortize expensive PDE solves by learning mappings from input geometries directly to solution fields. Moreover, prior work has shown that fully converged labels are unnecessary when training from Monte Carlo simulations: neural operators can recover accurate predictions from noisy, partially converged estimates~\citep{viswanath2026operator}. This enables training of neural operators with inexpensive Monte Carlo supervision while providing fast, smooth value-field predictions at inference time.

We propose \textbf{P}lanning \textbf{U}nder \textbf{N}oise (PUN), shown in Fig.~\ref{fig:title-pipeline}, an operator-learning framework that maps environment geometry to value fields that account for control noise. Boundary conditions can be understood through a heat analogy. With unit desirability at the goal and zero desirability on obstacle boundaries, the goal acts like a hot source, while the obstacle walls act like cold sinks. The zero-Dirichlet condition fixes the field to zero along each obstacle surface, making the entire obstacle boundary a heat sink. In a long or narrow passage, heat spreading from the goal therefore leaks into the walls along the route. The farther the goal's influence must travel, the more it is attenuated, leaving a weak signal in distant parts of free space that discretization, sampling, or learning errors can obscure. Narrow passages make this particularly pronounced because the absorbing walls remain close on both sides. This attenuation is visible in Fig.~\ref{fig:main-collage}(c), where the weak goal signal provides insufficient directional guidance for starting points far from the goal, making it difficult to navigate toward the target. Moreover, the screened Poisson equation already includes decay throughout free space through its screening term; absorbing obstacle boundaries introduce an additional source of attenuation. We therefore keep a Dirichlet condition at the goal but impose homogeneous Neumann (zero-flux) conditions on obstacle boundaries. The walls now act like insulation: they prevent heat from escaping through the boundary without forcing the desirability itself to vanish there. Heat can spread along a corridor without continually draining into its walls, allowing the goal's influence to reach farther into free space. Screening still causes decay, but the additional loss through obstacles is removed. This is precisely the effect illustrated in Fig.~\ref{fig:main-collage}(d): compared with the absorbing Dirichlet field in Fig.~\ref{fig:main-collage}(c), the mixed-boundary field extends the goal's influence much farther through the environment, preserving a stronger signal for planning along long routes. The Neumann condition also makes the value gradient parallel to smooth obstacle boundaries; therefore, gradient-based planning directs motion along walls rather than into them. This permits guidance along walls in narrow passages without the strong repulsion induced by zero-Dirichlet conditions. Under our regularity assumptions, the value field and control remain bounded at obstacles, whereas zero-Dirichlet desirability makes both diverge there (Appendix~\ref{app:proofs:reflection-choice}).

We first derive dimension-general Poisson kernels and Monte Carlo weights, extending mesh-free supervision beyond three dimensions for scalability. This enables supervision generation without a volumetric configuration-space mesh. We then train PUN using noisy, partially converged Monte Carlo estimates in the weak supervision paradigm. By amortizing the PDE solution across geometries, PUN produces spatially coherent value fields that guide gradient following and local search without requiring a fully converged Monte Carlo solve for every environment. We evaluate PUN across 2D--4D environments and compare its planning success and runtime with nominal planners and stochastic controllers under execution noise. In the experiments varying the execution noise level $\lambda$, we plan a path from the current state at each control step, follow it briefly, and repeat from the updated state; we refer to this as replanning. Our contributions can be summarized as follows:

\begin{itemize}[leftmargin=10pt]
    \item We formulate Poisson-based path planning as a mixed boundary-value problem, imposing a Dirichlet condition at the goal and homogeneous Neumann conditions on obstacle boundaries.

\item We introduce an operator-learning framework that maps environment geometry to noise-aware value fields, enabling zero-shot planning in unseen complex geometries without per-instance retraining or numerical PDE solves.

\item We generalize the mesh-free Monte Carlo solver WoSt beyond three dimensions by deriving the dimension-dependent Green's function and step weights for the screened Poisson equation.

\item We introduce a test-time scaling strategy that evaluates reachable candidate states on a local lattice using a fixed global learned field, without retraining, improving planning success over direct gradient following in the reported 2D--4D evaluations (Tables~\ref{tab:main-planning} and~\ref{tab:controller-dimensions}).

\end{itemize}

\begin{figure}[t]
\centering
\resizebox{1.0\linewidth}{!}{%
\input{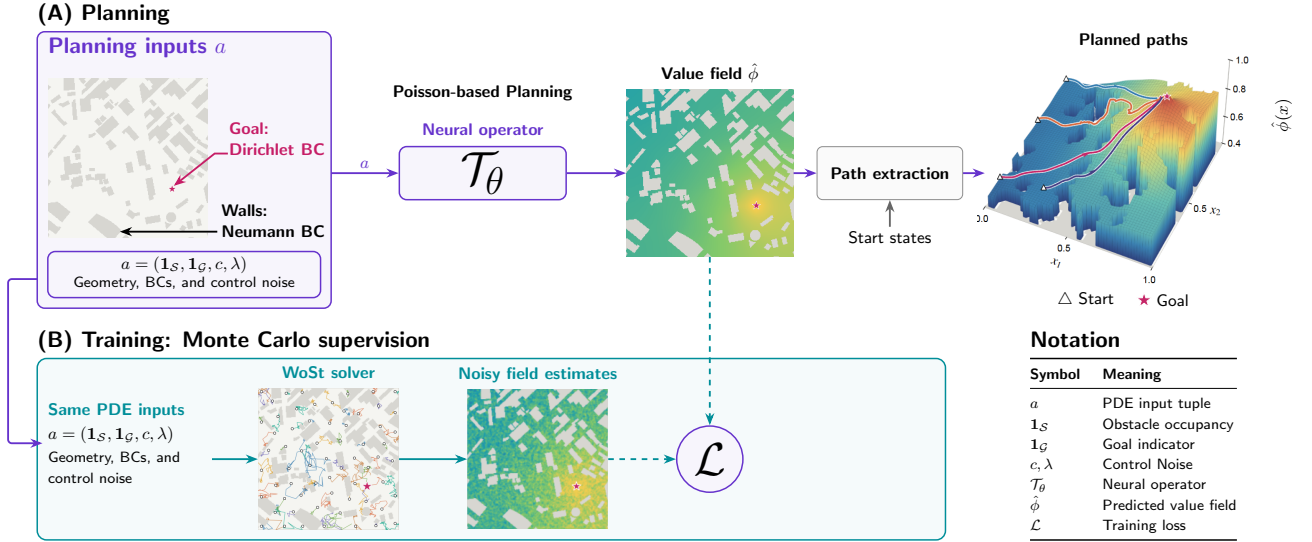}%
    }
\caption{Planning Under Noise: \textbf{(A) Planning:} a neural operator maps geometry, goal, cost, and noise inputs to a desirability field. Grid-based local search selects feasible routes with high endpoint desirability, executes a short prefix, and replans. \textbf{(B) Training:} Walk-on-Stars (WoSt) estimates provide pointwise supervision; neural caching supplies learned continuation at the walk horizon.}
\label{fig:title-pipeline}
\end{figure}
On 2D application benchmarks, zero-shot WoSt-supervised FNO reaches 93.62\% local-search success, up from 83.87\% with gradient following on the same fields (Table~\ref{tab:main-planning}). With the same local-search, FNO-based PUN achieves approximately $941\times$ and $58\times$ speedups over direct WoSt on the 2D zero-shot application benchmarks and 4D manipulator test set, respectively (Table~\ref{tab:infer-compact}). In 4D replanning, PUN achieves 95.00--100\% success at 10.51--11.66 ms/start across the tested noise levels (Table~\ref{tab:replanning}; Fig.~\ref{fig:main-collage}(a)). At $\lambda=4$, PUN is $3.29\times$ faster than RRT and $5.09$--$32.15\times$ faster than A$^\star$, Dijkstra, MPPI, and PI$^2$-CMA. FNO also maintains over 90\% local-search success across resolutions up to $4096\times4096$ ($256\times$ as many grid points as used in training) without retraining (Fig.~\ref{fig:main-collage}(b)).
\section{Preliminaries}
\label{sec:prelim}
In this section we derive the elliptic boundary value problem that our operator learns, and fix the notation used throughout.
 
\subsection{Stochastic Formulation of Path Planning}
\label{sec:prelim:stochastic}

\begin{figure}[t]
\begin{minipage}[t]{0.45\textwidth}
\begin{algorithm}[H]
\caption{\textsc{Walk}$(\xi,\Psi_{\theta^-})$, constant $\kappa$, max step $K$}
\label{alg:wost}
\footnotesize
\setlength{\baselineskip}{0.95\baselineskip}
\begin{algorithmic}[1]
\STATE {\bfseries In:} start point $\xi$; EMA target $\Psi_{\theta^-}$
\STATE $w\leftarrow1$
\FOR{$j=1$ {\bfseries to} $K$}
\STATE \textbf{break if} $\xi\in\Gamma_D$
\STATE $r\leftarrow r(\xi)$ \hfill $\triangleright$ Eq.~\ref{eq:radius}
\STATE Sample $(\mathbf v,\rho)$ as in Section~\ref{sec:method:step}
\STATE $w\leftarrow w\,Q_d^\kappa(\rho,r)$; $\xi\leftarrow\xi+\rho\mathbf v$
\ENDFOR
\IF{$\xi\in\Gamma_D$}
\STATE \textbf{return} $w\,g(\xi)$
\ENDIF
\STATE \textbf{return} $w\,\mathrm{sg}[\Psi_{\theta^-}(\xi)]$
\end{algorithmic}
\end{algorithm}
\end{minipage}\hfill
\begin{minipage}[t]{0.52\textwidth}
\begin{algorithm}[H]
\caption{Training with $m$ queries and $N\geq2$ walks per query}
\label{alg:train}
\footnotesize
\setlength{\baselineskip}{0.95\baselineskip}
\begin{algorithmic}[1]
\STATE {\bfseries In:} environments $\mathcal E$; $m,N,K$, loss weights, learning and EMA rates
\STATE Initialize $\theta$ and $\theta^-\leftarrow\theta$
\REPEAT
\STATE Sample $a\sim\mathcal E$, queries $x^i$, and goal points
\STATE Evaluate $\mathcal T_\theta[a]$ and detached $\mathcal T_{\theta^-}[a]$
\STATE $\hat{\Psi}^{i,n}\leftarrow\textsc{Walk}(\xi^i,\Psi_{\theta^-})$, independently
\STATE $\widehat\Psi_i\leftarrow N^{-1}\sum_n \hat{\Psi}^{i,n}$ \hfill $\triangleright$ Eq.~\ref{eq:truncated}
\STATE $\theta\leftarrow\theta-\alpha_{\mathrm{lr}}\nabla_\theta\mathfrak L(\theta)$
\STATE $\theta^-\leftarrow(1-\alpha_{\mathrm{ema}})\theta^-+\alpha_{\mathrm{ema}}\theta$
\UNTIL{training ends}
\STATE {\bfseries Inference:} control from Eq.~\ref{eq:param}
\end{algorithmic}
\end{algorithm}
\end{minipage}
\end{figure}

\paragraph{Control problem.}
Let $\Omega\subset\mathbb R^d$ and let $\xi\in\overline\Omega$ be the starting point. The controlled process $X^{\xi,\mathbf u}=(X_t^{\xi,\mathbf u})_{t\geq0}$ has interior dynamics
\begin{equation}
\mathrm{d}X^{\xi,\mathbf{u}}_t = \big(\mathbf{b}(X^{\xi,\mathbf{u}}_t) + G(X^{\xi,\mathbf{u}}_t)\,\mathbf{u}_t\big)\,\mathrm{d}t
+ \Sigma^{1/2}(X^{\xi,\mathbf{u}}_t)\,\mathrm{d}\mathbf{W}_t,
\qquad X^{\xi,\mathbf{u}}_0 = \xi,
\label{eq:sde}
\end{equation}
with control $\mathbf u_t\in\mathbb R^k$ and standard Brownian motion $\mathbf W$; the process reflects normally at obstacles and workspace boundaries ($\Gamma_N$) and stops at the goal ($\Gamma_D$). We minimize
\begin{equation}
J(x,\mathbf{u}) = \mathbb{E}\left[\int_0^{\tau_D^{\mathbf u}} \Big( q(X^{\xi,\mathbf{u}}_t) + \tfrac{1}{2}\mathbf{u}_t^\top R\,\mathbf{u}_t \Big) \mathrm{d}t \;+\; \phi(X^{\xi,\mathbf{u}}_{\tau_D^{\mathbf u}})  \middle|\xi=x\right],
\label{eq:cost}
\end{equation}
where $\tau_D^{\mathbf u}:=\inf\{t\geq0:X_t^{\xi,\mathbf u}\in\Gamma_D\}$ is the first goal hit, $q\geq0$, $R\succ0$, and $\phi\geq0$ is the terminal cost. Write $g=e^{-\phi/\lambda}$ for the Dirichlet desirability; operator learning fixes $\phi\equiv0$, hence $g\equiv1$, while the PDE identities retain general positive $g$. Where the value $V(x)=\inf_{\mathbf u}J(x,\mathbf u)$ is $C^2$ it satisfies the stationary HJB equation, which is nonlinear: eliminating the optimal control leaves a term quadratic in $\nabla V$ (Appendix~\ref{app:proofs:colehopf}, Eq.~\ref{eq:hjb-sub}).
 
\begin{assumption}[Control--noise matching]
\label{ass:matching}
There exists $\lambda > 0$ with $\Sigma(x) = \lambda\, G(x) R^{-1} G(x)^\top$ for all $x \in \Omega$.
\end{assumption}
This is the standard condition of linearly solvable stochastic optimal control~\citep{kappen2005,todorov2007linearly,todorov2009,theodorou2010generalized}: a direction that is cheap to actuate is also one along which the system is uncertain, and $\lambda$ sets the exchange rate between the two. This matching condition enables the exact linearization of the HJB equation used here. The Cole--Hopf change of variable $\Psi(x)\triangleq\exp(-V(x)/\lambda)$~\citep{hopf1950,cole1951,fleming1977exit} cancels the quadratic term and turns the HJB equation into a linear elliptic equation for $\Psi$, as shown in Appendix~\ref{app:proofs:colehopf}.
 
\paragraph{Path-planning specialization.}
For a fully actuated, passive system we take $\mathbf{b} = 0$, $G = I$, and $R = I$, so Assumption~\ref{ass:matching} forces $\Sigma = \lambda I$. Writing the state cost as $q(x) = \tfrac{1}{2}c(x)^2$ with $c > 0$, the linearized equation becomes the screened Poisson equation:
\begin{equation}
\Delta \Psi - \kappa^2(x)\,\Psi = 0,
\qquad
\kappa(x) \triangleq \frac{c(x)}{\lambda}.
\label{eq:screened}
\end{equation}
Throughout, we take $c$ constant within each environment, so $\kappa$ is constant.
 
\paragraph{Boundary conditions.}
Let $D$ be the workspace, $\mathcal S$ the obstacle interiors, and $\mathcal G$ an open goal region. The free domain is $\Omega=D\setminus(\overline{\mathcal S}\cup\overline{\mathcal G})$, assumed connected, with boundary partitioned into $\Gamma_D=\partial\mathcal G$ and $\Gamma_N=\partial\mathcal S\cup\partial D$. We impose
\begin{equation}
\Psi = g \ \text{ on } \Gamma_D \qquad\text{(absorbing goal)},
\qquad
\partial_n \Psi = 0 \ \text{ on } \Gamma_N \qquad\text{(reflecting)} .
\label{eq:bc}
\end{equation}
The goal is the only Dirichlet set; obstacles and the outer domain boundary are reflecting in the PDE model. Under Assumption~\ref{ass:reg} the mixed problem has a unique solution with $\Psi\in(0,1]$ and $\inf_{\overline\Omega}\Psi>0$, so $\log\Psi$ is bounded and the log parameterization is well defined (Proposition~\ref{prop:app-positivity}). See Appendix~\ref{app:proofs:reflection-choice} for a detailed analysis.

\subsection{Neural Operator}
\label{sec:prelim:no}
 
A neural operator learns a map from problem data to a solution field~\citep{li2021fno}. Here we learn the desirability operator
\begin{equation}
\mathcal{T} : \big(\ \underbrace{\mathbf{1}_{\mathcal{S}}}_{\text{obstacles}},\ \underbrace{\mathbf{1}_{\mathcal{G}}}_{\text{goal}},\ \underbrace{c}_{\text{cost}},\ \underbrace{\lambda}_{\text{noise}}\big) \;\longmapsto\; \Psi,
\label{eq:operator}
\end{equation}
Both geometric inputs are masks on the workspace: $\mathbf1_{\mathcal S}$ is the occupancy map of the obstacles and $\mathbf1_{\mathcal G}$ marks the goal region; goal data $g\equiv1$ is fixed. A single trained model $\mathcal T_\theta$ predicts the field for new environments. Its spatial gradients provide control through Eq.~\ref{eq:param}, and its values support local search (Section~\ref{sec:method:local-search}). 
\section{Related Works}
\label{sec:related}
\begin{figure}[t]
\centering
\includegraphics[width=1.0\linewidth]{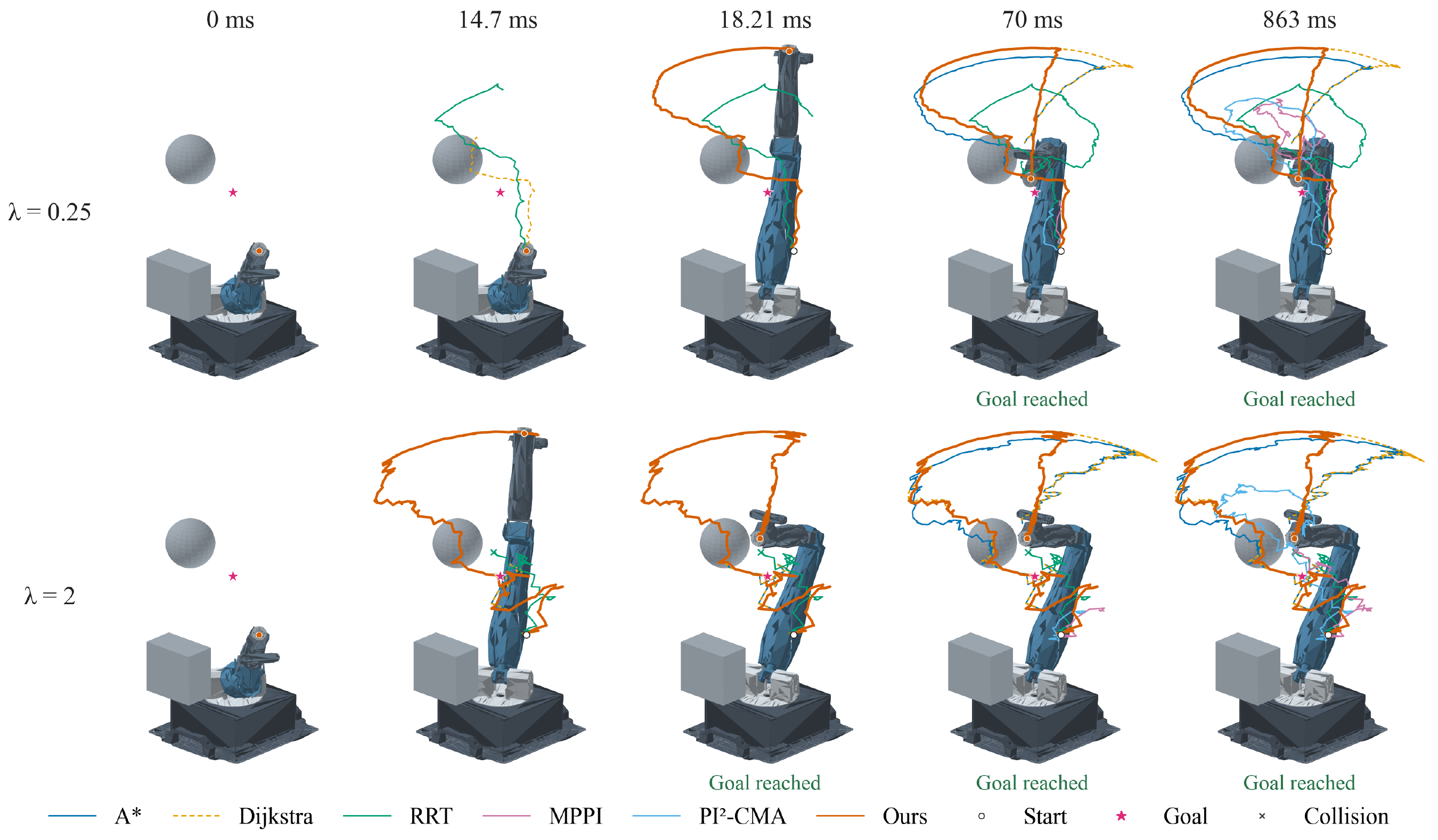}
\caption{Replanning trajectories on the same 4D manipulator, start, and goal at $\lambda=0.25$ (top) and $\lambda=2$ (bottom). Columns show successive execution states; PUN's trajectory is orange. PUN reaches the goal in both illustrated rollouts. Aggregate success and runtime are reported in Table~\ref{tab:replanning}.}
\label{fig:replanning-movingai-overlay}
\end{figure}

\paragraph{Path Planning.} Classical planners divide into two families. Sampling-based methods such as PRM~\citep{kavraki1996prm}, RRT~\citep{lavalle2001rrt}, and RRT$^\star$~\citep{karaman2011rrtstar} build a graph of collision-free motions, while search-based methods such as A$^\star$~\citep{hart1968astar}, Dijkstra~\citep{dijkstra1959}, and fast marching~\citep{sethian1996fastmarching} compute a cost-to-go over a discretization; both plan for a deterministic system and defer execution noise to tracking and replanning. Sampling-based methods such as MPPI~\citep{williams2017mppi,williams2017information} and PI$^2$-CMA~\citep{stulp2012pi2cma} optimize control sequences or policy parameters using sampled trajectory costs. Our implementations and their treatment of execution noise are described in Appendix~\ref{app:replanning-protocol}. We compare against these planners in Table~\ref{tab:replanning}. Planning Neural Operator (PNO)~\citep{matada2025pno} learns the solution operator of the Eikonal equation from labels produced by a deterministic solver, but it is limited to deterministic systems. NTFields~\citep{ni2023ntfields} and EikoNet~\citep{smith2020eikonet} train a physics-informed network on the Eikonal residual, which fits a single geometry per model and must be retrained for each new environment. In contrast, PUN is trained over a family of geometry classes with approximate Monte Carlo supervision across diverse noise ranges to support planning with different noise spectra while achieving reliable stability.
 
\paragraph{Monte Carlo Methods.}
Monte Carlo estimators for PDEs can be categorized depending on the type of PDE, include Walk-on-Spheres~\citep{sawhney2020wos}, Walk-on-Boundaries~\citep{sugimoto2023wob}, and Backward Stochastic Differential Equations (BSDE)~\citep{zhang2017backward}. Recent works combine Monte Carlo methods with neural networks to further reduce variance~\citep{nam2024nwos,li2023neural}, which train a network by regressing on noisy Monte Carlo targets and thereby amortize repeated pointwise Monte Carlo estimation into a forward pass. Two of their ingredients we adopt directly: a deterministic cap on the number of walk steps, and evaluation of the partially trained model when a walk exceeds that cap. The diffusion loss~\citep{nusken2021diffusion} truncates an SDE simulation and supplements it with a boundary physics loss.

\paragraph{Neural Operator.} Neural operators parameterize maps between function spaces and can be evaluated on different discretizations, with numerical accuracy depending on the representation~\citep{kovachki2023neural}. The architectures we compare differ in how they propagate information: FNO~\citep{li2021fno}, UNO~\citep{rahman2023uno}, LocalNO~\citep{liuschiaffini2024localno}, DISCO~\citep{ocampo2023disco}, GINO~\citep{li2023gino}, UPT~\citep{alkin2024upt} and Transolver~\citep{wu2024transolver}. We include U-Net~\citep{ronneberger2015unet} as a fixed-resolution reference. Physics-informed operators such as PINO~\citep{li2024pino} add a PDE residual to the objective; we use walk labels as the primary signal with regularization on gradient terms.

\section{Planning under Noise (PUN) with Neural Operator}
\label{sec:method}

PUN learns a desirability field from Walk-on-Stars supervision and uses the predicted field to guide local search. Labels are constructed directly in $\Psi$-space using finite-horizon walks, with learned continuation for unfinished trajectories. The neural operator is trained online on these labels. Algorithms~\ref{alg:wost} and~\ref{alg:train} describe a single walk and the training loop, respectively; walks and terminal evaluations are batched.

\subsection{The Walk-on-Stars Chain}
\label{sec:method:step}

\paragraph{Geometry and stopping.}
Let $\sigma=\inf\{k:\xi_k\in\Gamma_D\}$ be the first exact goal-boundary hit, possibly infinite. The ideal chain stops at an exact hit, where the terminal value is $g(\xi_\sigma)$; practical walks stop in a goal shell or at a finite horizon. For a spatial location $x$, the star radius is
\begin{equation}
\begin{gathered}
r(x)=\min\{d_D(x),\,d_{\mathrm{sil}}(x)\},\\
\mathrm{St}(x)=\text{the adjacent component of }B(x,r(x))\cap\Omega .
\end{gathered}
\label{eq:radius}
\end{equation}
Here $d_{\mathrm{sil}}$ is the distance to the closest visibility silhouette of $\Gamma_N$. An \emph{admissible star step} additionally requires that this component be star-shaped about $x$ and that its boundary consist only of the sphere and $\Gamma_N$. Appendix~\ref{app:wost-implementation} describes the shell stopping rule and radius floor used in place of exact absorption.

Write $\xi_k$ for the WoSt position after $k$ steps, with $\xi_0=\xi$ the starting point and $\xi_k$ random for $k\geq1$. Given $\xi_k$, construct $\mathrm{St}(\xi_k)$ with radius $r_k=r(\xi_k)$. Draw a direction uniformly on the sphere in the interior or on the inward hemisphere at a smooth reflecting boundary point, then move to the first positive hit of the sphere or an obstacle/domain boundary:
\begin{equation}
\mathbf v_k\sim\mathrm{Unif}(\mathbb S^{d-1}),\qquad
\rho_k=\text{first-hit distance},\qquad
\xi_{k+1}=\xi_k+\rho_k\mathbf v_k.
\label{eq:chain}
\end{equation}
The hemisphere restriction replaces the first distribution at $\xi_k\in\Gamma_N$. We assume measurable geometry queries and positive hit distances almost surely. 

On an admissible star, the direction law has surface density
\begin{equation}
p(x,\mathrm d\mathbf z)=
\frac{n_{\mathbf z}\cdot(\mathbf z-x)}{\alpha(x)\omega_d\|\mathbf z-x\|^d}\,\mathrm dS(\mathbf z),
\qquad \int_{\partial\mathrm{St}(x)}p(x,\mathrm d\mathbf z)=1,
\label{eq:kernel}
\end{equation}
where $\omega_d=2\pi^{d/2}/\Gamma(d/2)$ and $\alpha=1$ in the interior or $1/2$ at a smooth boundary pole. Tangencies have zero density. Where a step is not admissible, Eq.~\ref{eq:chain} still defines a law on first hits, but Eq.~\ref{eq:kernel} no longer describes it.

\paragraph{The estimator.}
Screening discounts each step by $W_k=Q_d^\kappa(\rho_k,r_k)\in(0,1)$, given in closed form for every $d$ in Appendix~\ref{app:proofs:Q}. For constant $\kappa$, Green's identity on each admissible active star gives
\begin{equation}
\mathbb E[W_k\Psi(\xi_{k+1})\mid\mathcal F_k]=\Psi(\xi_k),\qquad k<\sigma.
\label{eq:estimator}
\end{equation}
Multiplying by the preceding weights makes $M_k=(\prod_{j<k\wedge\sigma}W_j)\Psi(\xi_{k\wedge\sigma})$ a bounded martingale. This gives the finite-horizon identity below without requiring an exact goal hit in finitely many steps. The identity concerns the weighted ray sampler; its steps need not follow Brownian exit distributions.

\subsection{Label Construction}
\label{sec:method:labels}

For a finite walk cap $K$, write $\tau=\sigma\wedge K$. Repeated conditional expectation in Eq.~\ref{eq:estimator}, keeping the chain fixed after a goal hit, gives
\begin{equation}
\mathbb E_x\!\left[\Big(\prod_{j<\tau}W_j\Big)\Psi(\xi_\tau)\middle| \xi_0=x\right]=\Psi(x).
\label{eq:truncated}
\end{equation}
The terminal factor is $g(\xi_\sigma)$ on a goal hit and the true continuation $\Psi(\xi_K)$ otherwise. With true continuation and admissible steps, averaging $N$ independent replicas gives an unbiased label. Training substitutes learned continuation as described below. See Appendix~\ref{app:log-variant} for the log-space alternative and its bias correction.

Practical walks stop within a goal shell of width $\varepsilon$ and use a radius floor. Under the smooth-domain and admissibility assumptions, shell stopping with true continuation changes the mean by at most $C_\Psi\varepsilon$, where $C_\Psi=\sup_{\overline\Omega}\|\nabla\Psi\|$. The radius floor can violate admissibility and introduces an additional error of unknown sign~\citep{sawhney2023wost}. Appendix~\ref{app:wost-implementation} gives the implementation and the finite-horizon shell bound.

\paragraph{Bootstrapping.}
At the horizon $\Psi$ is not available, so training substitutes a detached EMA copy of the model. The resulting bias is that copy's own error at the endpoint, weighted by the product of the step weights over the $K$ steps, so it is discounted wherever steps are long compared with the screening length.

\subsection{Operator Parameterization and Objective}
\label{sec:method:param}

For $\lambda>0$ and a positive differentiable prediction, the desirability field defines the feedback control through its logarithm:\footnote{With our choice $R=G=I$, the feedback reduces to $\mathbf u_\theta=(\lambda/\Psi_\theta)\nabla\Psi_\theta$. For a positive differentiable field, $\nabla\log\Psi_\theta=\nabla\Psi_\theta/\Psi_\theta$: the gradients differ only by a positive, state-dependent scalar and have identical normalized directions wherever the gradient is nonzero.}
\begin{equation}
\Psi_\theta=\mathcal T_\theta[a],\qquad
V_\theta=-\lambda\log\Psi_\theta,\qquad
\mathbf u_\theta=\lambda R^{-1}G^\top\nabla\log\Psi_\theta.
\label{eq:param}
\end{equation}
For a coordinate-differentiable decoder, the gradient is obtained by automatic differentiation. A grid-output backbone requires a specified differentiable interpolant or a spatial differentiation rule. The input $a$ is defined in Eq.~\ref{eq:operator}, with $g\equiv1$ fixed.

We combine value and goal supervision with gradient loss to train fields for downstream planning:
\begin{equation}
\mathcal L_{\mathrm{val}}=\mathbb E\big[(\Psi_\theta-\widehat\Psi)^2\big],\quad
\mathcal L_{\mathcal G}=\mathbb E\big[(\Psi_\theta-g)^2\big],\quad
\mathcal L_{\mathrm{grad}}=\frac{\mathbb E\big[\|\nabla\Psi_\theta-\nabla\widehat\Psi\|^2\big]}
{\mathrm{sg}\big[\mathbb E\big[\|\nabla\widehat\Psi\|^2\big]\big]}.
\label{eq:loss-grad}
\end{equation}
The target gradient $\nabla\widehat\Psi$ is computed by finite differences of the label field. The displayed gradient loss normalization assumes positive target gradient energy.

The objective is their weighted sum,
\begin{equation}
\mathfrak L(\theta)=\mathcal L_{\mathrm{val}}
+\beta_{\mathcal G}\,\mathcal L_{\mathcal G}
+\beta_{\mathrm{grad}}\,\mathcal L_{\mathrm{grad}}.
\label{eq:loss}
\end{equation}
The goal loss carries the exact boundary value $g$ rather than a walk label. The gradient term is included because a small value error need not give a small gradient error, and the control in Eq.~\ref{eq:param} is computed from the gradient: $\mathcal L_{\mathrm{grad}}$ normalizes gradient error by target gradient energy, treating the denominator as fixed during differentiation. Direct-$\Psi$ supervision avoids log-transform bias and weights absolute field error, giving small absolute errors in low-desirability regions less weight than log-space supervision. Appendix~\ref{app:log-variant} reports the log-space alternative, in which errors are measured on a relative scale.

\subsection{Grid-Based Local Search}
\label{sec:method:local-search}

PUN supports test-time scaling by increasing the local search budget without modifying the learned global field. More depth lets the search evaluate more reachable endpoints, while using the field for guidance. The same endpoint search can be run on other fields; in PUN, it is run on a global noise-aware field that is reused across replans.

Paths are extracted from $\Psi_\theta$ on a lattice. Let $\mathcal X_h$ be the free cells at spacing $h$. The feasible neighbors of $x$ are $\mathcal N(x)=\{x'\in\mathcal X_h:\|x'-x\|_2=h\}$, the axis-adjacent free cells. Write $\mathcal R_k(x)$ for the states reachable from $x$ in at most $k$ such moves, including $x$ itself. A plan is a feasible lattice path $\pi=(x_0,\ldots,x_\ell)$ with $\ell\leq L$, scored by the predicted desirability $\Psi_\theta(x_\ell)$ at its endpoint.
For $\lambda>0$ and $\Psi_\theta>0$, maximizing this score is equivalent to minimizing $V_\theta=-\lambda\log\Psi_\theta$ at the endpoint. The score itself is well defined for any finite prediction.

A controller is specified by a depth $L$ and commit length $e\leq L$. With one edge committed per replan, the next move is
\begin{equation}
x_1\in\operatorname*{arg\,max}_{y\in\mathcal N(x_0)}\ \max_{z\in\mathcal R_{L-1}(y)}\Psi_\theta(z),
\label{eq:local-move}
\end{equation}
a step towards the largest desirability reachable through a candidate neighbor. For $e>1$, a path maximizing endpoint desirability is constructed, its first $\min(e,\ell)$ edges are executed, and the search is repeated from the new state. Execution stops early upon reaching the goal. See Appendix~\ref{app:exp-protocol} for further details.

A key idea is that evaluating reachable endpoints permits an intermediate decrease in desirability when it leads to a better state, so local fluctuations need not determine the next move. This might happen when the field is not smooth and there are local valleys, so local search allows jumps over imperfect predictions. Appendix~\ref{app:proofs:grid} derives the graph-search rule and its large-depth limit, and Section~\ref{sec:exp} evaluates planning performance.

Optionally, local search can include a soft wall penalty,
\[
C_r(\pi)
=
\sum_{j=1}^{\ell}
\left[
\frac{\big(a_r(x_j)-d_{\mathrm{wall}}(x_j)/h\big)_+}{r}
\right]^2,
\]
where $(u)_+=\max(u,0)$, $r$ is the preferred clearance in cells,
and $d_{\mathrm{wall}}$ measures distance to obstacles or the
outer boundary. Normally $a_r=r$; near a goal reachable within
the local depth, it decreases with the remaining free-edge
distance to that goal. Among paths retaining at least half the
best positive local field gain, we prefer the smallest penalty. Narrow
passages and starts or goals near walls are still possible.
\raggedbottom
\section{Experiments}
\label{sec:exp}

We evaluate how backbone, supervision, training objective, and controller affect planning, then test transfer across grid resolutions and replanning under execution noise. Table~\ref{tab:main-planning} compares controllers on the same fields, while Table~\ref{tab:replanning} compares complete planners under common execution rules. Models are pretrained on generated environments and evaluated on unseen maps without fine-tuning. The main 2D application suite contains 592 maps and 37,888 starts across robotics, terrain, marine, urban, and game environments (Appendix~\ref{app:benchmark-names}). Success measures goal-reaching rate; relative $L_2$ measures field error against FEM references. Unless specified, local search (LS) denotes local search (10). Bold marks the best value in each comparison, including ties. Experiments use NVIDIA B300 and B200 GPUs; training and evaluation protocols are in Appendix~\ref{app:exp-protocol}.

\begin{table}[t]
\centering
\setlength{\abovecaptionskip}{0pt}
\RaggedRight
\caption{2D application-benchmark evaluation of neural backbones trained with WoSt supervision and direct WoSt inference: planning success and relative $L_2$ error against the FEM ground-truth field.}
\label{tab:main-planning}
{\small\setlength{\tabcolsep}{3pt}
\renewcommand{\arraystretch}{0.93}
\begin{tabular*}{\linewidth}{@{\extracolsep{\fill}}lrrrr@{}}
\hline
Model & Gradient (\%) & Local search (4) (\%) & Local search (10) (\%) & Relative $L_2$ \\
\hline
UPT & 2.48 & 2.78 & 3.57 & 0.5533 \\
Transolver & 3.26 & 5.34 & 15.30 & 0.5158 \\
WoSt & 2.71 & 6.30 & 36.45 & 0.1852 \\
GINO & 46.56 & 51.24 & 53.11 & 0.2628 \\
U-Net & 74.21 & 84.22 & 88.98 & 0.1605 \\
DISCO & 82.49 & 87.60 & 91.08 & 0.1041 \\
LocalNO & 80.34 & 86.63 & 91.35 & 0.0915 \\
U-NO & 78.08 & 86.11 & 91.68 & 0.0792 \\
\textbf{FNO} & \textbf{83.87} & \textbf{91.25} & \textbf{93.62} & \textbf{0.0615} \\
\hline
\end{tabular*}}
\end{table}

\paragraph{Backbone and amortization.} Table~\ref{tab:main-planning} compares FNO~\citep{li2021fno}, U-NO~\citep{rahman2023uno}, U-Net\footnote{U-Net inputs are rescaled to its fixed resolution.}~\citep{ronneberger2015unet}, GINO~\citep{li2023gino}, LocalNO~\citep{liuschiaffini2024localno}, DISCO~\citep{ocampo2023disco}, UPT~\citep{alkin2024upt}, and Transolver~\citep{wu2024transolver}, all trained with WoSt supervision. FNO has the highest success under all three controllers and the lowest relative $L_2$ error (0.0615); its local-search (10) success is 93.62\%, compared with 91.68\% for the next-best U-NO. These rankings apply to the tested implementations under the training-compute policy in Appendix~\ref{app:exp-protocol}. Direct WoSt inference reaches only 2.71\% gradient and 36.45\% local-search (10) success, versus 83.87\% and 93.62\% for FNO. Thus, at the reported solver settings, learning from approximate WoSt targets yields more useful planning fields than direct inference. We use FNO in the remaining ablations.

\paragraph{Supervision.} With FNO fixed, Table~\ref{tab:supervision-main} compares WoSt targets, Deep Ritz training~\citep{e2018deepritz}, and PDE-residual training with PINO~\citep{li2024pino}. WoSt achieves 93.62\% local-search success, versus 91.46\% for Deep Ritz and 59.69\% for PINO; gradient success follows the same ordering (83.87\%, 81.41\%, and 32.50\%). Deep Ritz has slightly lower field error than WoSt (0.0608 versus 0.0615), but lower planning success. Field error is therefore not a reliable indicator of how useful these learned fields are for control.

\begin{table}[t]
\centering
\renewcommand{\arraystretch}{0.93}
\begin{minipage}[t]{0.32\linewidth}
\vspace{0pt}
\setlength{\abovecaptionskip}{0pt}
\RaggedRight
\begin{minipage}[t][3\baselineskip][t]{\linewidth}
\caption{FNO supervision on 2D application benchmarks: success and relative error.}
\label{tab:supervision-main}
\end{minipage}\par
\centering
{\fontsize{8}{10}\selectfont\setlength{\tabcolsep}{1pt}
\begin{tabular*}{\linewidth}{@{\extracolsep{\fill}}lrrr@{}}
\hline
Source & Grad. (\%) & LS (\%) & Rel. err. \\
\hline
WoSt & \textbf{83.87} & \textbf{93.62} & 0.0615 \\
Deep Ritz & 81.41 & 91.46 & \textbf{0.0608} \\
PINO & 32.50 & 59.69 & 0.5644 \\
\hline
\end{tabular*}}
\end{minipage}\hfill
\begin{minipage}[t]{0.32\linewidth}
\vspace{0pt}
\setlength{\abovecaptionskip}{0pt}
\RaggedRight
\begin{minipage}[t][3\baselineskip][t]{\linewidth}
\caption{Walk-horizon ablation on the generated 2D test set: planning success (\%).}
\label{tab:walk-cap-main}
\end{minipage}\par
\centering
{\fontsize{8}{10}\selectfont\setlength{\tabcolsep}{1pt}
\begin{tabular*}{\linewidth}{@{\extracolsep{\fill}}rrr@{}}
\hline
Steps & Grad. (\%) & LS (\%) \\
\hline
10 & 32.18 & 51.89 \\
100 & \textbf{40.55} & 66.80 \\
400 & 38.06 & \textbf{69.69} \\
800 & 37.84 & 69.66 \\
\hline
\end{tabular*}}
\end{minipage}\hfill
\begin{minipage}[t]{0.32\linewidth}
\vspace{0pt}
\setlength{\abovecaptionskip}{0pt}
\RaggedRight
\begin{minipage}[t][3\baselineskip][t]{\linewidth}
\caption{FNO loss ablations on 2D application benchmarks: planning success (\%).}
\label{tab:loss-2d-bench}
\end{minipage}\par
\centering
{\fontsize{8}{10}\selectfont
\setlength{\tabcolsep}{1pt}
\begin{tabular*}{\linewidth}{@{\extracolsep{\fill}}lrr@{}}
\hline
Model & Grad. (\%) & LS (\%) \\
\hline
Full model & \textbf{83.87} & \textbf{93.62} \\
\textcolor{red}{$-$} value loss & 67.71 & 83.26 \\
\textcolor{red}{$-$} goal loss & 82.93 & 93.11 \\
\textcolor{red}{$-$} gradient loss & 80.17 & 92.04 \\
\hline
\end{tabular*}}
\end{minipage}
\end{table}

\paragraph{Walk horizon and learned continuation.} In Table~\ref{tab:walk-cap-main}, we vary the walk cap using a smaller training budget of 16 maps, eight goals per map, four trajectories per query, and 40 epochs, evaluated on generated 2D maps. Local-search success increases from 51.89\% at 10 steps to 69.69\% at 400 steps, but is 69.66\% at 800 steps. Gradient success instead peaks at 40.55\% with 100 steps. Longer walks therefore do not uniformly improve planning: the horizon changes both sampled propagation and reliance on the learned continuation (Section~\ref{sec:method:labels}). Further details of this study are in Appendix~\ref{sec:exp:labels}. These success rates should not be compared directly with the full-budget application results.

\paragraph{Training objective.} In Table~\ref{tab:loss-2d-bench}, we remove one loss term at a time from Eq.~\ref{eq:loss}. Removing value supervision causes the largest decrease: gradient success falls from 83.87\% to 67.71\%, and local-search success from 93.62\% to 83.26\%. Removing goal or gradient supervision gives smaller local-search decreases, to 93.11\% and 92.04\%, respectively. All three terms help in this setting, with value supervision making the largest measured contribution. These ablations use direct desirability supervision.

\paragraph{Path extraction.} On the same learned FNO fields, Table~\ref{tab:main-planning} shows success rising from 83.87\% with gradient following to 91.25\% with local search (4) and 93.62\% with local search (10), a total gain of 9.75 percentage points. This test-time refinement requires no retraining. In 2D, the two local-search variants use $(L,e)=(5,1)$ and $(10,3)$, so both depth and committed prefix length change. Local search (10) also exceeds gradient success for every other backbone, although it reaches only 36.45\% on direct WoSt fields. Looking beyond immediate neighbors can therefore recover useful routes, but does not eliminate dependence on field quality. Figure~\ref{fig:local-search-cells} illustrates the two search neighborhoods on one fixed field; controller details are in Appendix~\ref{app:exp-protocol}.

\paragraph{Transfer across discretizations.} In Fig.~\ref{fig:main-collage}(b), we test models trained at $256\times256$ on grids up to $4096\times4096$, without retraining. FNO maintains 90.62--96.25\% local-search (10) success across the five resolutions and reaches 92.50\% at $4096\times4096$, compared with 71.25\% for U-NO and 57.81\% for U-Net. U-NO leads at $512\times512$ (95.31\% versus 94.06\% for FNO); FNO leads at the other four resolutions. GINO runs out of memory above $256\times256$. The quantitative comparison supports FNO's transfer across these discretizations; Fig.~\ref{fig:resolution-main} separately illustrates predicted fields and paths at three resolutions.

\FloatBarrier

\paragraph{Replanning under execution noise.} Table~\ref{tab:replanning} compares PUN with A$^\star$~\citep{hart1968astar}, Dijkstra~\citep{dijkstra1959}, RRT~\citep{lavalle2001rrt}, MPPI~\citep{williams2017mppi}, and PI$^2$-CMA~\citep{stulp2012pi2cma} on 2D and 3D application benchmarks and 4D manipulators. Obstacle hits count as failures (Appendix~\ref{app:replanning-protocol}). PUN has the highest reported success in all 12 dimension--noise settings, tying A$^\star$ and Dijkstra at $\lambda=0.25$ in 4D. At $\lambda=4$, PUN exceeds the strongest baseline, RRT, by 9.00 percentage points in 2D (45.67\% versus 36.67\%), 4.00 in 3D (16.00\% versus 12.00\%), and 9.17 in 4D (95.00\% versus 85.83\%). PUN also has the lowest reported runtime throughout: 9.55--14.51 ms/start in 2D, 14.30--15.64 in 3D, and 10.51--11.66 in 4D. At $\lambda=4$ in 4D, its 11.66 ms/start is 3.29 times faster than the fastest baseline, RRT (38.31 ms/start); Fig.~\ref{fig:main-collage}(a) summarizes mean 4D runtimes. The advantage is relative: PUN's 3D success still falls from 92.00\% to 16.00\% as noise increases. These comparisons concern the reported implementations and settings. Figure~\ref{fig:replanning-movingai-overlay} shows PUN reaching the goal in two illustrative 4D rollouts at $\lambda=0.25$ and $2$.

\begin{table}[!t]
\centering
\caption{Replanning under execution noise on 2D and 3D application benchmarks and 4D manipulators. Each entry reports success (\%) / runtime (ms/start). Bold marks the highest success and lowest runtime within each dimension and noise level, including ties.}
\label{tab:replanning}
\label{tab:replanning-2d}
\label{tab:replanning-3d}
\label{tab:replanning-4d}
{\footnotesize
\renewcommand{\arraystretch}{0.93}
\setlength{\tabcolsep}{4pt}
\begin{tabular*}{\linewidth}{@{\extracolsep{\fill}}llrrrr@{}}
\hline
Dim. & Method & $\lambda=0.25$ & $\lambda=1$ & $\lambda=2$ & $\lambda=4$ \\
\hline
\multirow{6}{*}{\textbf{2D}} & A* & 59.00 / 43.63 & 42.33 / 47.09 & 35.00 / 51.18 & 23.67 / 53.76 \\
 & Dijkstra & 59.00 / 154.94 & 42.33 / 174.95 & 35.00 / 203.41 & 23.67 / 239.97 \\
 & RRT & 82.00 / 98.45 & 68.67 / 103.87 & 52.33 / 112.27 & 36.67 / 124.71 \\
 & MPPI & 84.00 / 1734.11 & 61.00 / 1773.31 & 45.67 / 1838.73 & 31.67 / 2126.12 \\
 & PI$^2$-CMA & 78.33 / 1330.76 & 60.00 / 1484.91 & 48.00 / 1524.70 & 35.00 / 1606.82 \\
 & PUN (zero-shot) & \textbf{91.67} / \textbf{9.55} & \textbf{74.00} / \textbf{11.14} & \textbf{62.33} / \textbf{14.07} & \textbf{45.67} / \textbf{14.51} \\
\hline
\multirow{6}{*}{\textbf{3D}} & A* & 36.67 / 73.40 & 23.33 / 60.47 & 12.00 / 73.85 & 6.00 / 65.25 \\
 & Dijkstra & 37.33 / 73.46 & 22.67 / 63.04 & 12.00 / 76.99 & 6.00 / 67.73 \\
 & RRT & 53.33 / 238.31 & 35.33 / 169.90 & 25.33 / 122.85 & 12.00 / 126.23 \\
 & MPPI & 62.67 / 532.31 & 34.67 / 577.84 & 18.67 / 554.11 & 7.33 / 704.31 \\
 & PI$^2$-CMA & 47.33 / 452.25 & 22.00 / 435.01 & 21.33 / 475.98 & 8.67 / 606.55 \\
 & PUN (zero-shot) & \textbf{92.00} / \textbf{15.64} & \textbf{68.67} / \textbf{15.34} & \textbf{40.00} / \textbf{14.30} & \textbf{16.00} / \textbf{14.43} \\
\hline
\multirow{6}{*}{\textbf{4D}} & A* & \textbf{100.00} / 16.81 & 96.67 / 17.08 & 72.50 / 78.02 & 59.17 / 62.44 \\
 & Dijkstra & \textbf{100.00} / 17.17 & 96.67 / 17.34 & 72.50 / 66.75 & 59.17 / 59.39 \\
 & RRT & 98.33 / 24.20 & 95.83 / 25.80 & 90.00 / 43.00 & 85.83 / 38.31 \\
 & MPPI & 95.00 / 411.24 & 92.50 / 398.18 & 81.67 / 364.22 & 71.67 / 374.83 \\
 & PI$^2$-CMA & 96.67 / 379.79 & 92.50 / 357.69 & 86.67 / 355.44 & 78.33 / 356.72 \\
 & PUN & \textbf{100.00} / \textbf{10.63} & \textbf{100.00} / \textbf{10.51} & \textbf{100.00} / \textbf{10.55} & \textbf{95.00} / \textbf{11.66} \\
\hline
\end{tabular*}
}
\end{table}

\FloatBarrier

\section{Discussion}

\paragraph{Extensions.}
Several extensions are possible. A \emph{spatially varying cost} $c(x)$ is the most useful: with constant $c$, the state-cost term penalizes elapsed time and the objective also penalizes control effort. Reflection carries no additional contact cost; choosing $c(x)$ as a function of obstacle clearance would add a clearance-dependent running cost. \emph{Transient problems} replace the stationary equation by its time-dependent counterpart, giving a finite-horizon value function rather than a first-exit one, which suits tasks with deadlines or moving goals. \emph{Partial information} is the most direct use of our amortization: since the operator maps geometry to a field in one forward pass, replanning as a robot discovers obstacles requires only a forward pass on the updated geometry instead of a new solve. The open question is how to condition the operator on a partially observed occupancy map.

\paragraph{Limitations.}
The value field is derived for reflecting boundaries. Local search allows its use in collision-failure tasks, but nominal graph feasibility does not guarantee collision-free execution under noise. Deriving the exact value function for that execution model would require a corresponding boundary condition. Assumption~\ref{ass:matching} is a further restriction. The control--noise matching assumption enables exact HJB linearization and is standard in classical path-integral control~\citep{kappen2005,theodorou2010generalized}. Our experiments are also limited to settings where this assumption is reasonable: we test geometric planning problems in up to four dimensions and do not include full robot models with contact. Evaluating on humanoid robots with contact would test both the assumption and the capacity of the operator. Finally, adversarial geometries such as mazes remain challenging, since long corridors can lead to weak or poorly resolved learned gradients.
\section{Conclusion}
We introduced Planning Under Noise (PUN), a neural operator framework that learns noise-aware planning fields from approximate Walk-on-Stars supervision. Dimension-general screening kernels support mesh-free supervision beyond three dimensions, and local search converts predicted fields into paths on unseen geometries. Experiments across 2D--4D environments demonstrate zero-shot planning and effective learning from stochastic labels. In replanning comparisons under execution noise, PUN achieves the lowest reported runtime among the tested methods, while matching or exceeding the success of every compared planner at each tested noise level across 2D--4D.
\subsection*{AI use statement}

In this work, we used generative AI tools to assist with parts of method implementation. Additionally, we used generative AI tools to polish writing and proofs and identify relevant literature. We did not use generative AI tools for other applicable tasks requiring disclosure. All AI-assisted work was reviewed and verified by the authors. We take responsibility for the final content of this work, including text, claims, and artifacts produced with the aid of generative AI.

\section*{Acknowledgements}
Anima Anandkumar and Xi Deng are both supported in part by the Bren endowed chair and the AI2050 senior fellow program at Schmidt Sciences. Hong Chul Nam is supported by a PhD fellowship for the Center for Data Science at NYU.

\begingroup\small\setlength{\bibsep}{2pt}
\bibliographystyle{pun}
\bibliography{references}

\endgroup

\appendix
\makeatletter
\setlength{\@fptop}{0pt}
\setlength{\@fpsep}{\baselineskip}
\setlength{\@fpbot}{0pt plus 1fil}
\makeatother
\newpage
\section{Notation Table}
Table~\ref{tab:notation} collects the shared symbols used across the paper.

{\small
\setlength{\tabcolsep}{5pt}
\renewcommand{\arraystretch}{1.18}
\begin{longtable}{llp{0.52\linewidth}}
\caption{Notation, with the equation, assumption, or section where each symbol is defined.}
\label{tab:notation}\\
\hline
Symbol & Definition & Meaning \\
\hline
\endfirsthead
\multicolumn{3}{l}{\small\emph{Table~\ref{tab:notation}, continued}}\\
\hline
Symbol & Definition & Meaning \\
\hline
\endhead
\hline
\endfoot

\multicolumn{3}{l}{\emph{Control problem}}\\
$\xi,\ x$ & \S\ref{sec:prelim:stochastic} & starting point of the process or walk; conditioned as $\mid\xi=x$ where the start is not otherwise fixed \\
$X^{\xi,\mathbf{u}}_t$ & Eq.~\ref{eq:sde} & controlled process at time $t$, from the prescribed start $\xi$ under control $\mathbf{u}$ \\
$\mathbf{b}$ & Eq.~\ref{eq:sde} & drift of the controlled system ($\mathbf{b}=0$ in the specialization) \\
$G,\ R$ & Eq.~\ref{eq:sde}, Eq.~\ref{eq:cost} & control matrix and control penalty ($G=R=I$ in the specialization) \\
$\Sigma$ & Eq.~\ref{eq:sde} & diffusion tensor; $\Sigma=\lambda I$ when $G=R=I$ under matching \\
$q,\ c$ & Eq.~\ref{eq:cost}, Eq.~\ref{eq:screened} & state cost $q=\tfrac12c^2$, with $c>0$ the local slowness \\
$\phi,\ g$ & Eq.~\ref{eq:cost}, \S\ref{sec:prelim:stochastic} & terminal cost and its transform $g=e^{-\phi/\lambda}\in(0,1]$; $g\equiv1$ by default \\
$\lambda$ & Ass.~\ref{ass:matching} & temperature; sets the noise--control exchange rate \\
$V,\ \Psi$ & Eq.~\ref{eq:cost}, \S\ref{sec:prelim:stochastic} & value function and desirability $\Psi=e^{-V/\lambda}$ \\
\hline
\multicolumn{3}{l}{\emph{Geometry and the elliptic problem}}\\
$\mathcal G,\ \Gamma_D$ & Eq.~\ref{eq:bc} & open goal region and its absorbing boundary $\Gamma_D=\partial\mathcal G$ \\
$\Gamma_N$ & Eq.~\ref{eq:bc} & obstacles $\partial\mathcal{S}$ and outer boundary, reflecting \\
$\kappa$ & Eq.~\ref{eq:screened} & screening coefficient $\kappa=c/\lambda$ \\
$\omega_d$ & Eq.~\ref{eq:app-notation} & total solid angle, $2\pi^{d/2}/\Gamma(d/2)$ \\
$\alpha(x)$ & Eq.~\ref{eq:kernel} & fraction of solid angle subtended by $\Omega$ at $x$; $1$ inside, $\tfrac12$ on $\Gamma_N$ \\
\hline
\multicolumn{3}{l}{\emph{The walk}}\\
$\xi_k$ & Eq.~\ref{eq:chain} & WoSt point after $k$ steps, from $\xi_0=\xi$; random for $k\geq1$ \\
$r(x),\ \mathrm{St}(x)$ & Eq.~\ref{eq:radius} & radius and star centered at the current spatial location $x$; $r_k=r(\xi_k)$ \\
$d_D,\ d_{\mathrm{sil}}$ & Eq.~\ref{eq:radius} & distance to $\Gamma_D$; distance to the nearest visibility silhouette \\
$\rho_k$ & Eq.~\ref{eq:chain} & distance travelled by step $k$; $\rho_k\leq r_k$ \\
$\sigma$ & \S\ref{sec:method:step} & first exact goal-hit step $\inf\{k:\xi_k\in\Gamma_D\}$, possibly infinite; practical walks use shell stopping and a horizon \\
$W_k$ & Eq.~\ref{eq:app-Qgen} & step weight $Q^{\kappa}_d(\rho_k,r_k)$ \\
$\gamma_d$ & Eq.~\ref{eq:Q-reduces} & full-sphere weight $Q^{\kappa}_d(r,r)$ \\
\hline
\multicolumn{3}{l}{\emph{Labels and early stopping}}\\
$\widehat\Psi$ & Eq.~\ref{eq:truncated} & Monte Carlo estimate \\
$K,\ N,\ m$ & \S\ref{sec:method:labels}, Alg.~\ref{alg:train} & step budget; walks per query point; query points per gradient step \\
\hline
\multicolumn{3}{l}{\emph{Planning lattice}}\\
$h,\ \mathcal X_h$ & \S\ref{sec:method:local-search} & lattice spacing and free grid cells \\
$L,\ e$ & \S\ref{sec:method:local-search} & depth and number of edges committed per replan \\
$\mathcal N,\ \mathcal R_k$ & \S\ref{sec:method:local-search} & feasible neighbors of a state; states reachable in at most $k$ moves \\
$B_k$ & Eq.~\ref{eq:local-propagation} & largest desirability reachable within $k$ moves \\
\hline
\multicolumn{3}{l}{\emph{Operator and training}}\\
$\mathcal{T}$ & Eq.~\ref{eq:operator} & exact solution operator $a\mapsto\Psi$, which $\mathcal T_\theta$ approximates \\
$\mathcal{T}_\theta,\ \theta^-$ & Eq.~\ref{eq:param}, Alg.~\ref{alg:train} & the learned operator $a\mapsto\Psi_\theta$; its EMA target parameters \\
$\mathfrak{L}$ & Eq.~\ref{eq:loss} & training objective \\
\hline
\multicolumn{3}{l}{\emph{Kernels}}\\
$\nu$ & Eq.~\ref{eq:app-Qgen} & Bessel order $\tfrac d2-1$ \\
$G^{\kappa,B}$ & Eq.~\ref{eq:app-ball-general} & screened Green's function of the ball \\
$P^{\mathbb{R}^d}$ & Eq.~\ref{eq:app-Pfactor} & unscreened free-space Poisson kernel, the sampling density \\
\hline
\end{longtable}
}

\section{Screened Poisson Kernels in \texorpdfstring{$d$}{d} Dimensions}
\label{app:kernels}
 
This appendix derives the kernels used by the estimator of Section~\ref{sec:method:step} in general dimension. Throughout, $d \geq 2$, $\kappa > 0$, and
\begin{equation}
\nu := \frac{d}{2} - 1,
\qquad
r := \|x-\mathbf{y}\|,
\qquad
\omega_d = \frac{2\pi^{d/2}}{\Gamma(d/2)} .
\label{eq:app-notation}
\end{equation}
In Appendix~\ref{app:kernels:free} the radial variable is the source separation $r = \|x-\mathbf{y}\|$; from Appendix~\ref{app:kernels:ball} onward $r$ is the radius of the ball or star and $\rho := \|x-\mathbf{y}\|$ denotes the source separation. We seek Green's functions for the screened Laplacian
\begin{equation*}
\mathcal{L}_\kappa := -\Delta + \kappa^2,
\qquad
\mathcal{L}_\kappa\,G(x,\mathbf{y}) = \delta(x-\mathbf{y}).
\end{equation*}
Away from the pole every case below solves the same radial ODE,
\begin{equation}
u''(r) + \frac{d-1}{r}\,u'(r) - \kappa^2 u(r) = 0,
\qquad
\text{basis }\ \big\{\, r^{-\nu}I_\nu(\kappa r),\ \ r^{-\nu}K_\nu(\kappa r) \,\big\},
\label{eq:app-radial}
\end{equation}
with $I_\nu$ regular at the origin and $K_\nu$ singular. Cases differ only in which combination the outer boundary condition selects; the overall constant is fixed once and for all by the flux jump at $r = 0$.
 
\subsection{Free Space}
\label{app:kernels:free}

In free space the radial solution of Eq.~\ref{eq:app-radial} that decays at infinity is the $K_\nu$ branch, and the flux jump at the pole fixes its constant:
\begin{equation}
G_{\mathbb{R}^d}(r) = \frac{\kappa^{\nu}}{(2\pi)^{d/2}}\,\frac{K_\nu(\kappa r)}{r^{\nu}}
\label{eq:app-free}
\end{equation}

As a check, for $d = 2$ ($\nu = 0$) this gives $G = K_0(\kappa r)/2\pi$, and for $d = 3$ ($\nu = \tfrac12$), using $K_{1/2}(z) = \sqrt{\pi/2z}\,e^{-z}$ gives the Yukawa potential $G = e^{-\kappa r}/4\pi r$.

\paragraph{Radial derivative.}
The Neumann part of $\partial\mathrm{St}$ requires the normal derivative of Eq.~\ref{eq:app-free}. Using $\tfrac{d}{dz}\big[z^{-\nu}K_\nu(z)\big] = -z^{-\nu}K_{\nu+1}(z)$,
\begin{equation*}
\partial_r G_{\mathbb{R}^d}(r) = -\frac{\kappa^{\nu+1}}{(2\pi)^{d/2}}\,\frac{K_{\nu+1}(\kappa r)}{r^{\nu}} \;<\; 0
\end{equation*}
which is strictly negative for all $r > 0$, so the kernel is monotonically decreasing in every dimension.

\subsection{The Ball: Dirichlet, Neumann, and Robin}
\label{app:kernels:ball}
 
Let $B_R = \{\|x\| < R\}$ with the pole at the center, as in every walk step. The solution is radial and takes the form
\begin{equation}
G_{B_R}(r) = \frac{\kappa^{\nu}}{(2\pi)^{d/2}}\,\frac{1}{r^{\nu}}\Big[ K_\nu(\kappa r) - A\, I_\nu(\kappa r) \Big],
\qquad 0 < r < R,
\label{eq:app-ball-general}
\end{equation}
with the same prefactor as Eq.~\ref{eq:app-free}, since the flux jump at the origin is unchanged by the boundary. Only $A$ depends on the boundary condition. Writing $z = \kappa r$ and using the scaled-derivative identities
\begin{equation}
\frac{\mathrm{d}}{\mathrm{d}z}\big[z^{-\nu} I_\nu(z)\big] = z^{-\nu} I_{\nu+1}(z),
\qquad
\frac{\mathrm{d}}{\mathrm{d}z}\big[z^{-\nu} K_\nu(z)\big] = -z^{-\nu} K_{\nu+1}(z),
\label{eq:app-scaled}
\end{equation}
the radial derivative of Eq.~\ref{eq:app-ball-general} is proportional to $-K_{\nu+1}(z) - A\,I_{\nu+1}(z)$. The two cases we use follow:
\begin{align}
\text{Dirichlet } G(R) = 0: && A_D &= \frac{K_\nu(\kappa R)}{I_\nu(\kappa R)},
\label{eq:app-AD}\\
\text{Neumann } \partial_r G(R) = 0: && A_N &= -\frac{K_{\nu+1}(\kappa R)}{I_{\nu+1}(\kappa R)} .
\label{eq:app-AN}
\end{align}
Note the order shift in $A_N$: since $\tfrac{\mathrm{d}}{\mathrm{d}z}[z^{-\nu}K_\nu] = -z^{-\nu}K_{\nu+1}$ and likewise for $I_\nu$, imposing $\partial_r G(R)=0$ on $r^{-\nu}[K_\nu - A I_\nu]$ produces the order-$(\nu+1)$ ratio rather than the naive $K'_\nu/I'_\nu$, for which the derivative vanishes only when $\nu = 0$.

\subsection{The Screened Poisson Kernel and the Step Weight}
\label{app:kernels:poisson}
 
The walk uses the Poisson kernel, $P^\kappa_{B_R}(\mathbf{0},\mathbf{z}) = -\partial_{n_{\mathbf{z}}} G^D_{B_R}$ for $\mathbf{z} \in \partial B_R$, which by radial symmetry is constant on the sphere. Its total mass has a closed form, and it is exactly the factor $\gamma_d$ appearing in Eq.~\ref{eq:Q-reduces}.
 
\begin{proposition}[Total mass of the screened Poisson kernel]
\label{prop:app-gamma}
Let $u$ solve $\Delta u - \kappa^2 u = 0$ in $B_R$ with $u \equiv 1$ on $\partial B_R$. Then
\begin{equation*}
u(r) = \Big(\frac{R}{r}\Big)^{\nu}\, \frac{I_\nu(\kappa r)}{I_\nu(\kappa R)},
\qquad
\int_{\partial B_R} P^\kappa_{B_R}(\mathbf{0},\mathbf{z})\, \mathrm{d}S(\mathbf{z}) \;=\; u(\mathbf{0}) \;=\; \gamma_d(\kappa R),
\end{equation*}
where the full-sphere weight is
\begin{equation}
\gamma_d(z) := \frac{(z/2)^{\nu}}{\Gamma(\nu+1)\,I_\nu(z)} = Q^{\kappa}_d(r,r)\big|_{z=\kappa r} .
\label{eq:Q-reduces}
\end{equation}
Consequently
\begin{equation}
P^\kappa_{B_R}(\mathbf{0},\mathbf{z}) = \frac{\gamma_d(\kappa R)}{\omega_d\,R^{d-1}}
= \frac{\Gamma(d/2)\,\gamma_d(\kappa R)}{2\pi^{d/2} R^{d-1}} .
\label{eq:app-poisson}
\end{equation}
\end{proposition}
\begin{proof}
Regularity at the origin selects the $I_\nu$ branch of Eq.~\ref{eq:app-radial}, and the stated normalization gives $u(R) = 1$. Taking $r \to 0$ with $I_\nu(z) \sim (z/2)^\nu/\Gamma(\nu+1)$,
\begin{equation*}
u(\mathbf{0}) = \frac{R^\nu (\kappa/2)^\nu}{\Gamma(\nu+1)\, I_\nu(\kappa R)} = \frac{(\kappa R/2)^{\nu}}{\Gamma(\nu+1)\, I_\nu(\kappa R)} = \gamma_d(\kappa R).
\end{equation*}
The mass identity follows since $u(\mathbf{0}) = \int P^\kappa\, u\, \mathrm{d}S$ with $u \equiv 1$ on the boundary, and Eq.~\ref{eq:app-poisson} follows from constancy on the sphere.
\end{proof}
 
Proposition~\ref{prop:app-gamma} is the $\rho = r$ case of the star-step weight Eq.~\ref{eq:app-Qgen}, and Appendix~\ref{app:proofs:Q} shows the general weight reduces to it. It also gives the two asymptotic regimes. Expanding $I_\nu$ and using $4(\nu+1) = 2d$,
\begin{equation*}
\gamma_d(z) = \Big[1 + \frac{z^2}{2d} + O(z^4)\Big]^{-1} = 1 - \frac{z^2}{2d} + O(z^4),
\qquad
\gamma_d(z) \sim \frac{(z/2)^{\nu}\sqrt{2\pi z}}{\Gamma(\nu+1)}\, e^{-z}\ \ (z \to \infty),
\end{equation*}
so a short step is nearly weight-preserving with deficit $\kappa^2R^2/2d$, while a step much longer than the screening length is damped exponentially; hence steps beyond a few screening lengths contribute negligibly.
 
\subsection{The Screened Ball Poisson Kernel and the Star-Step Weight}
\label{app:kernels:Q}
 
Appendix~\ref{app:kernels:poisson} treated the Poisson kernel at the center of a ball, which corresponds to a full-length sphere step. A star step may instead terminate against $\Gamma_N$ at a distance $\rho < r$, so we need the kernel at interior points of the ball.
 
Following \citet{sawhney2023wost}, the screened ball Green's functions are
\begin{align*}
G^{\kappa,B}_{2}(x,\mathbf{y}) &= \frac{1}{2\pi}\Big[K_0(\kappa\rho) - I_0(\kappa\rho)\tfrac{K_0(\kappa r)}{I_0(\kappa r)}\Big],
&
G^{\kappa,B}_{3}(x,\mathbf{y}) &= \frac{1}{4\pi}\,\frac{\sinh\big((r-\rho)\kappa\big)}{\rho\,\sinh(\kappa r)},
\end{align*}
with $\rho = \|\mathbf{y}-x\|$, which agree with the general-$d$ expression Eq.~\ref{eq:app-ball-general} at $d = 2,3$. Both vanish at $\rho = r$, which is the cancellation that removes the unknown-flux term on the spherical part of the boundary.
 
The corresponding Poisson kernels factor as a scalar times the \emph{free-space} kernel,
\begin{equation}
P^{\kappa,B}_d(x,\mathbf{z}) = Q^{\kappa}_d(\rho,r)\, P^{\mathbb{R}^d}(x,\mathbf{z}),
\qquad
P^{\mathbb{R}^d}(x,\mathbf{z}) = \frac{n_{\mathbf{z}}\cdot(\mathbf{z}-x)}{\omega_d\,\rho^{d}},
\label{eq:app-Pfactor}
\end{equation}
with $Q^{\kappa}_d$ given for all $d$ by Eq.~\ref{eq:app-Qgen}.
 
\begin{lemma}[General-dimension step weight]
\label{lem:app-Qgen}
For every $d \geq 2$ and $0 < \rho \leq r$,
\begin{equation}
Q^{\kappa}_d(\rho,r) := \frac{\partial_\rho G^{\kappa,B}}{\partial_\rho G^{B}}
= \frac{2}{\Gamma(\nu+1)}\Big(\frac{\kappa\rho}{2}\Big)^{\nu+1}
\left[K_{\nu+1}(\kappa\rho) + \frac{K_\nu(\kappa r)}{I_\nu(\kappa r)}I_{\nu+1}(\kappa\rho)\right],
\qquad \nu = \tfrac{d}{2}-1,
\label{eq:app-Qgen}
\end{equation}
and $Q^{\kappa}_d(r,r) = \gamma_d(\kappa r)$.
\end{lemma}
 
\begin{proof}
Write $z = \kappa\rho$ and $\alpha = K_\nu(\kappa r)/I_\nu(\kappa r)$, so that by Eq.~\ref{eq:app-ball-general} and Eq.~\ref{eq:app-AD},
\[G^{\kappa,B} = \kappa^{2\nu}(2\pi)^{-d/2}\big[z^{-\nu}K_\nu(z) - \alpha\,z^{-\nu}I_\nu(z)\big].\]
The scaled-derivative identities Eq.~\ref{eq:app-scaled} give
\begin{equation*}
\partial_\rho G^{\kappa,B} = -\frac{\kappa^{\nu+1}}{(2\pi)^{d/2}}\,\rho^{-\nu}\big[K_{\nu+1}(\kappa\rho) + \alpha\,I_{\nu+1}(\kappa\rho)\big],
\end{equation*}
using $2\nu+1 - \nu = \nu+1$ for the power of $\kappa$. For $d>2$, the denominator has $G^{B} = c_d(\rho^{2-d} - r^{2-d})$ with $c_d = \Gamma(\nu)/(4\pi^{\nu+1})$, hence
$\partial_\rho G^{B} = -2\nu\,c_d\,\rho^{-2\nu-1} = -\Gamma(\nu+1)(2\pi^{\nu+1})^{-1}\rho^{-2\nu-1}$,
where $2\nu\Gamma(\nu) = 2\Gamma(\nu+1)$. Dividing the preceding display by this and collecting the powers of $2$ and $\pi$ gives Eq.~\ref{eq:app-Qgen}. Setting $\rho = r$ makes the bracket $\big[I_\nu K_{\nu+1} + I_{\nu+1}K_\nu\big]/I_\nu$, and the Bessel identity $I_\nu(z)K_{\nu+1}(z) + I_{\nu+1}(z)K_\nu(z) = 1/z$ reduces it to $1/(zI_\nu(z))$, leaving $(z/2)^\nu/(\Gamma(\nu+1)I_\nu(z)) = \gamma_d(z)$. For $d=2$, use $G^B(\rho)=(2\pi)^{-1}\log(r/\rho)$ and $\partial_\rho G^B=-1/(2\pi\rho)$ directly; the same ratio follows. The sphere reduction follows from the standard Wronskian $I_0K_1 + I_1K_0 = 1/z$. We verified Eq.~\ref{eq:app-Qgen} against the tabulated $d=2,3$ forms and the reduction against $\gamma_d$ at $d \in \{2,3,4,5,7,10\}$ to fourteen digits.
\end{proof}

Specializing Eq.~\ref{eq:app-Qgen} to the two dimensions tabulated by \citet{sawhney2023wost},
\begin{align}
Q^{\kappa}_2(\rho,r) &= \Big[K_1(\kappa\rho) + I_1(\kappa\rho)\,\tfrac{K_0(\kappa r)}{I_0(\kappa r)}\Big]\,\kappa\rho,
\label{eq:app-Q2}\\
Q^{\kappa}_3(\rho,r) &= e^{-\kappa\rho}\big(\kappa\rho+1\big) + \big(\kappa\rho\cosh(\kappa\rho) - \sinh(\kappa\rho)\big)\,\tfrac{e^{-\kappa r}}{\sinh(\kappa r)} .
\label{eq:app-Q3}
\end{align}

Table~\ref{tab:app-Q} gives representative values. Because of this factorization, the estimator is inexpensive: the geometric part $P^{\mathbb{R}^d}$ is exactly the signed solid angle subtended by $\partial\mathrm{St}$ at $\mathbf{z}$, so sampling directions uniformly gives unit importance weights, and screening enters only through the scalar $Q^{\kappa}_d$.
 
\begin{table}[!htbp]
\caption{Star-step weight $Q^{\kappa}_3(\rho,r)$ at $r = 1$. The weight is $1$ for a zero-length step and decreases with the realized hit distance $\rho$, reaching $\gamma_3(\kappa r) = \kappa r/\sinh(\kappa r)$ at $\rho = r$. A step cut short by an obstacle loses less weight than a full sphere step.}
\label{tab:app-Q}
\centering
{\small
\setlength{\tabcolsep}{6pt}
\begin{tabular}{lrrrrrrr}
\hline
$\rho/r$ & $0$ & $1/6$ & $1/3$ & $1/2$ & $2/3$ & $5/6$ & $1$ \\
\hline
$\kappa r = 1$ & 1.0000 & 0.9880 & 0.9593 & 0.9232 & 0.8880 & 0.8614 & 0.8509 \\
$\kappa r = 4$ & 1.0000 & 0.8558 & 0.6157 & 0.4086 & 0.2629 & 0.1766 & 0.1466 \\
\hline
\end{tabular}}
\end{table}
 
\subsection{Why a Ball Kernel Is Admissible on a Star Domain}
\label{app:kernels:transfer}

The kernels of Appendices~\ref{app:kernels:free}--\ref{app:kernels:Q} are constructed for free space and for a ball, yet Eq.~\ref{eq:app-admissible} integrates them over $\partial\mathrm{St}$, which is not a sphere. The following identity justifies this.

\begin{proposition}[Admissibility]
\label{prop:app-admissible}
Let $A \subseteq \Omega$ be a bounded Lipschitz domain with $x\in A$, let $\Psi$ solve $\mathcal{L}_\kappa\Psi = 0$ on $A$, and let $G$ be \emph{any} function with $(-\Delta+\kappa^2)G(x,\cdot) = \delta_{x}$ on a neighborhood of $\overline{A}$, so that $x$ is its only singularity in $\overline{A}$. Then with $P := -\partial_n G$ and $n$ the outward normal,
\begin{equation}
\Psi(x) = \int_{\partial A}\Big[P(x,\mathbf{z})\,\Psi(\mathbf{z}) + G(x,\mathbf{z})\,\partial_n\Psi(\mathbf{z})\Big]\mathrm{d}S(\mathbf{z}).
\label{eq:app-admissible}
\end{equation}
\end{proposition}

\begin{proof}
Green's second identity on $A$ for the pair $(\Psi, G)$ with the operator $\mathcal{L}_\kappa$ gives
\begin{equation*}
\int_A \big(\Psi\,\mathcal{L}_\kappa G - G\,\mathcal{L}_\kappa\Psi\big) = \int_{\partial A}\big(G\,\partial_n\Psi - \Psi\,\partial_n G\big)
\end{equation*}
the zeroth-order terms $\kappa^2\Psi G$ canceling identically between the two products. The second integrand on the left vanishes by hypothesis, and the first collapses to $\Psi(x)$ by the defining property of $G$. Rearranging gives Eq.~\ref{eq:app-admissible}.
\end{proof}

The chain's one-step identity is the case $G=G^{\kappa,B}$, the screened Dirichlet Green's function of the ball $B=B(x,r)$, applied on an admissible star $A=\mathrm{St}(x)$. Since $G^{\kappa,B}$ vanishes on the spherical part of $\partial A$ and $\partial_n\Psi=0$ on the reflecting part, the flux term drops and $\Psi(x)=\int_{\partial A}P^{\kappa,B}(x,\mathbf z)\Psi(\mathbf z)\,\mathrm dS(\mathbf z)$ with $P^{\kappa,B}=-\partial_nG^{\kappa,B}$. At a smooth boundary pole $x\in\Gamma_N$, the same steps on $A$ minus a half-ball of radius $\eta$ about $x$, with $\eta\to0$, give the identity with $\Psi(x)$ replaced by $\tfrac12\Psi(x)$, because the small hemisphere carries half of the singular flux; this is the factor $\alpha$ in Eq.~\ref{eq:kernel}. The direction law is the solid-angle density $P^{B}\,\mathrm dS/\alpha$ with $P^{B}=-\partial_nG^{B}$, and the weight is the flux ratio $Q_d^\kappa=P^{\kappa,B}/P^{B}$ of Eq.~\ref{eq:app-Q-flux}, so $\mathbb E[W_k\Psi(\xi_{k+1})\mid\xi_k=x]=\alpha(x)^{-1}\int_{\partial A}P^{\kappa,B}\Psi\,\mathrm dS=\Psi(x)$, the one-step property of Section~\ref{sec:method:step}. Multiplying by the earlier weights gives the martingale property before $\sigma$; the stopped sequence is constant afterwards.

The hypotheses require neither that $G$ be the Green's function \emph{of} $A$ nor that it satisfy any boundary condition on $\partial A$. All that is needed is a fundamental solution whose singularity inside $A$ is the single pole at $x$. Since $\mathrm{St}(x,r) \subseteq B(x,r)$, the ball kernel $G^{\kappa,B}$ has that property, and so do the free-space kernel and the ball kernel of any radius $r'\geq r$. Each of these kernels gives a valid representation of $\Psi(x)$, so we must choose among them.

The choice is governed by the unknown flux $\partial_n\Psi$ in Eq.~\ref{eq:app-admissible}. On $\partial\mathrm{St}$ that flux is zero on $\partial\mathrm{St}_N\subseteq\Gamma_N$ but unknown on $\partial\mathrm{St}_B\subseteq\partial B(x,r)$, which lies in the interior of $\Omega$ where no boundary condition applies. A usable identity therefore needs a kernel that vanishes on $\partial\mathrm{St}_B$ while remaining finite on $\partial\mathrm{St}_N$, and that is exactly the ball's Dirichlet kernel: $G^{\kappa,B}(x,\cdot)\equiv0$ on $\partial B(x,r)\supseteq\partial\mathrm{St}_B$, and it is finite and nonzero at the interior points constituting $\partial\mathrm{St}_N$. The free-space kernel is admissible but unsuitable here: $G^{\kappa,\mathbb{R}^d}>0$ on $\partial\mathrm{St}_B$, so the unknown flux term does not vanish.

The walk nonetheless samples directions by the \emph{free-space} solid angle. For the Laplace operator the two Poisson kernels agree on all of $\partial\mathrm{St}$, because the boundary corrector $G^{B}-G^{\mathbb{R}^d}$ for a centered pole is the constant radial homogeneous solution and so has zero normal derivative; directional sampling then importance-samples the kernel exactly and the step weight is $1$. Screening replaces that constant by the $I_\nu$ branch of Eq.~\ref{eq:app-radial}, and the resulting $\rho$-dependent corrector is precisely the second term of Eq.~\ref{eq:app-Qgen}. This explains why $Q^\kappa_d$ depends on the star radius $r$ and not only on the hit distance $\rho$: the $\rho$-dependence comes from the free-space fundamental solution, and the $r$-dependence comes from the ball corrector, whose normal derivative is nonzero only when $\kappa>0$.

Proposition~\ref{prop:app-admissible} is exact, so the ball-on-star transfer contributes no error of its own. One further term is specific to the identity: for nonzero Neumann data the second integrand of Eq.~\ref{eq:app-admissible} requires $G^{\kappa,B}$ at interior points of $B$, which is available in closed form but is estimated with a separate sample, since reusing the walk direction gives an estimator that is zero on flat boundaries and high-variance elsewhere~\citep{sawhney2023wost}. With $h\equiv0$ that term drops entirely, which is the case throughout this paper.

\subsection{Dimension-Specific Forms and the Laplace Limit}
\label{app:kernels:cases}
 
\begin{table}[!htbp]
\caption{Screened Poisson kernels by dimension. All entries follow from Eq.~\ref{eq:app-free}, Eq.~\ref{eq:app-ball-general}, and Eqs.~\ref{eq:app-AD}--\ref{eq:app-AN} with $\nu = d/2 - 1$.}
\label{tab:app-kernels}
\centering
{\small
\renewcommand{\arraystretch}{1.9}
\setlength{\tabcolsep}{4pt}
\begin{tabular}{cccc}
\hline
$d$ & $G_{\mathbb{R}^d}(r)$ & $G^D_{B_R}(r)$ & $\gamma_d(z)$ \\
\hline
$2$ & $\dfrac{K_0(\kappa r)}{2\pi}$ & $\dfrac{1}{2\pi}\!\left[K_0(\kappa r) - \dfrac{K_0(\kappa R)}{I_0(\kappa R)}I_0(\kappa r)\right]$ & $\dfrac{1}{I_0(z)}$ \\
$3$ & $\dfrac{e^{-\kappa r}}{4\pi r}$ & $\dfrac{1}{4\pi r}\!\left[e^{-\kappa r} - e^{-\kappa R}\dfrac{\sinh(\kappa r)}{\sinh(\kappa R)}\right]$ & $\dfrac{z}{\sinh z}$ \\
general & $\dfrac{\kappa^{\nu}}{(2\pi)^{d/2}}\dfrac{K_\nu(\kappa r)}{r^{\nu}}$ & $\dfrac{\kappa^{\nu}}{(2\pi)^{d/2} r^{\nu}}\!\left[K_\nu(\kappa r) - \dfrac{K_\nu(\kappa R)}{I_\nu(\kappa R)}I_\nu(\kappa r)\right]$ & $\dfrac{(z/2)^{\nu}}{\Gamma(\nu+1) I_\nu(z)}$ \\
\hline
\end{tabular}}
\end{table}
 
Table~\ref{tab:app-kernels} collects the dimension-specific forms. For $d>2$, $K_\nu(z)\sim2^{\nu-1}\Gamma(\nu)z^{-\nu}$ gives the unscreened limit
\begin{equation*}
G_{\mathbb{R}^d}(r) \;\longrightarrow\; \frac{\Gamma(d/2 - 1)}{4\pi^{d/2}}\, r^{2-d},
\qquad
\gamma_d(\kappa R) \;\longrightarrow\; 1,
\end{equation*}
which is $1/(4\pi r)$ in three dimensions. In two dimensions the free-space Green's function diverges by an additive constant as $\kappa\to0$; subtracting that constant gives $-(2\pi)^{-1}\log r$. In the Dirichlet ball kernel the constants cancel, giving $(2\pi)^{-1}\log(R/r)$. In all dimensions $Q_d^\kappa(\rho,r)\to1$ and $\gamma_d\to1$, recovering the unscreened weights. A purely Neumann unscreened problem needs a separate compatibility and normalization condition.

\section{Proofs}
\label{app:proofs}
 
\subsection{Standing Regularity Assumptions}
\label{app:proofs:reg}
 
\begin{assumption}[Regularity]
\label{ass:reg}
The free domain $\Omega$ is bounded and connected with $C^{2,\alpha}$ boundary, for $0<\alpha<1$. Its Dirichlet and Neumann sets are disjoint unions of boundary components, each with positive surface measure. Hence they are separated by a positive distance. This excludes Dirichlet--Neumann junctions, near which the solution can have singular gradients. The goal data satisfies $g\in C^{2,\alpha}(\Gamma_D)$ and $0<g_{\min}\leq g\leq1$. The cost satisfies $c\in C^\alpha(\overline\Omega)$ and $0<c_{\min}\leq c\leq c_{\max}$. The noise parameter $\lambda>0$ is fixed for each problem.
\end{assumption}
 
Under Assumption~\ref{ass:reg}, the mixed elliptic problem has a unique $\Psi\in C^{2,\alpha}(\overline\Omega)$~\citep{gilbarg2001elliptic}. These regularity statements concern the ideal smooth domain. Voxel edges and corners are not covered by this assumption; the smooth-boundary factor $\alpha=1/2$ and global gradient bounds do not automatically extend to them. 
The exact PDE identities for the chain additionally require constant screening and admissible stars; regularity does not guarantee that the radius-clamped sampler is exact. Write $\mathcal L_\kappa=-\Delta+\kappa^2$.

\subsection{Exact Linearization}
\label{app:proofs:colehopf}

Where $V$ is $C^2$ it satisfies the interior stationary HJB equation
\begin{equation*}
0 = \min_{\mathbf{u}} \Big\{ \nabla V^\top (\mathbf{b} + G\mathbf{u}) + q + \tfrac{1}{2}\mathbf{u}^\top R\,\mathbf{u} \Big\} + \tfrac{1}{2}\operatorname{tr}\!\big(\Sigma \nabla^2 V\big),
\end{equation*}
whose inner minimization is attained at $\mathbf{u}^\star = -R^{-1}G^\top \nabla V$. Substituting that minimizer leaves a quadratic nonlinearity in $\nabla V$,
\begin{equation}
0 = \nabla V^\top \mathbf{b} - \tfrac{1}{2}\nabla V^\top G R^{-1} G^\top \nabla V + q + \tfrac{1}{2}\operatorname{tr}\!\big(\Sigma \nabla^2 V\big),
\label{eq:hjb-sub}
\end{equation}
and this term cancels under the desirability transform.

\begin{lemma}[Cole--Hopf linearization is exact precisely under matching]
\label{lem:colehopf}
Let $V \in C^2$ satisfy Eq.~\ref{eq:hjb-sub} and set $\Psi = e^{-V/\lambda}$. Then
\begin{equation}
\tfrac{1}{2}\operatorname{tr}(\Sigma\nabla^2\Psi) + \mathbf{b}^\top\nabla\Psi - \tfrac{q}{\lambda}\Psi
= \frac{1}{2\Psi}\,\nabla\Psi^\top\Big[\Sigma - \lambda\, GR^{-1}G^\top\Big]\nabla\Psi .
\label{eq:app-ch-residual}
\end{equation}
In particular the right side vanishes for every $\Psi$ with $\nabla\Psi \neq 0$ if and only if $\Sigma = \lambda GR^{-1}G^\top$ as quadratic forms, which is Assumption~\ref{ass:matching}.
\end{lemma}
 
\begin{proof}
From $\Psi = e^{-V/\lambda}$ we get $\nabla V = -\lambda\nabla\Psi/\Psi$ and
$\nabla^2 V = -\lambda\frac{\nabla^2\Psi}{\Psi} + \lambda\frac{\nabla\Psi\otimes\nabla\Psi}{\Psi^2}$
Substitute into Eq.~\ref{eq:hjb-sub} term by term. The diffusion term becomes
\begin{equation*}
\tfrac{1}{2}\operatorname{tr}(\Sigma\nabla^2 V) = -\frac{\lambda}{2}\frac{\operatorname{tr}(\Sigma\nabla^2\Psi)}{\Psi} + \frac{\lambda}{2}\frac{\nabla\Psi^\top \Sigma\nabla\Psi}{\Psi^2}
\end{equation*}
and the quadratic control term becomes
\begin{equation*}
-\tfrac{1}{2}\nabla V^\top GR^{-1}G^\top\nabla V = -\frac{\lambda^2}{2}\frac{\nabla\Psi^\top GR^{-1}G^\top \nabla\Psi}{\Psi^2}.
\end{equation*}
Adding the two quadratic contributions gives
\begin{equation}
\frac{\lambda}{2\Psi^2}\,\nabla\Psi^\top\big[\Sigma - \lambda GR^{-1}G^\top\big]\nabla\Psi,
\label{eq:app-cancel}
\end{equation}
which is Eq.~\ref{eq:app-ch-residual} up to the stated factor. The remaining terms are $-\lambda \mathbf{b}^\top\nabla\Psi/\Psi + q$ and $-\tfrac{\lambda}{2}\operatorname{tr}(\Sigma\nabla^2\Psi)/\Psi$; multiplying the whole equation by $-\Psi/\lambda$ yields the stationary backward Kolmogorov equation
\begin{equation}
\tfrac{1}{2}\operatorname{tr}\!\big(\Sigma\nabla^2 \Psi\big) + \mathbf{b}^\top \nabla \Psi - \frac{q}{\lambda}\Psi = 0
\label{eq:kolmogorov}
\end{equation}
exactly when Eq.~\ref{eq:app-cancel} vanishes.
 
For universal cancellation of the transformed differential expression, put $E=\Sigma-\lambda GR^{-1}G^\top$. The symmetric quadratic form $v^\top E v$ vanishes for every $v$ if and only if $E=0$. Smooth positive test fields can have an arbitrary gradient at a point, so matching is necessary for this operator identity. For a particular solution $\Psi$, the residual can vanish even when matching fails.
\end{proof}
 
For the fully actuated reflected model under Assumption~\ref{ass:reg}, let $\Psi$ solve the mixed problem and set $V=-\lambda\log\Psi$ and $\mathbf u^*=-\nabla V$. Assume admissible controls, including $\mathbf u^*$, give well-defined dynamics with finite expected goal-hitting time and control energy. Completing the HJB square and applying It\^o's formula up to $\tau_D\wedge T$, then taking $T\to\infty$, yields
\begin{equation}
J(x,\mathbf u)=V(x)+\frac12\mathbb E_x\!\int_0^{\tau_D}\|\mathbf u_t-\mathbf u^*(X_t)\|^2\,\mathrm dt.
\label{eq:verification}
\end{equation}
The Neumann condition removes the reflection term; bounded $V$ and $V=\phi$ at the goal justify the stopped limit. Thus $J(x,\mathbf u)\geq V(x)$, with equality for $\mathbf u^*$, proving optimality.
 
\subsection{Positivity and Bounded Log-Desirability}
\label{app:proofs:positivity}

Part (ii) below fixes the scale on which $\log\Psi$ varies. For constant $c$ the screening length is $\lambda/c$, but additional length scales arise from the domain size, goal size, and obstacle geometry, so exponential decay with distance holds only asymptotically.

\begin{proposition}[Positivity and range]
\label{prop:app-positivity}
Under Assumption~\ref{ass:reg}:
\begin{enumerate}
\item[(i)] $0<\Psi\leq1$ on $\overline\Omega$, with $\Psi_{\min}=\min\Psi>0$;
\item[(ii)] for the fixed annulus $a<|x|<D$, $0<a<D$, constant $\kappa$, $\Psi(a)=1$, and $\Psi'(D)=0$, the log range is
$-\log\Psi(D)=\kappa(D-a)+O(1)$ as $\kappa\to\infty$.
\end{enumerate}
\end{proposition}
\begin{proof}
For (i), the maximum principle excludes an interior maximum above one or a nonpositive interior minimum. At a reflecting boundary, the Hopf boundary lemma~\citep{gilbarg2001elliptic} excludes either extremum unless the solution is constant. A constant violating the bounds contradicts the positive Dirichlet data. Continuity and compactness then give a strictly positive minimum.

For (ii), put $\nu=d/2-1$. The radial solution is
\begin{equation*}
\Psi(r)=\frac{r^{-\nu}[K_\nu(\kappa r)+A I_\nu(\kappa r)]}
{a^{-\nu}[K_\nu(\kappa a)+A I_\nu(\kappa a)]},\qquad
A=\frac{K_{\nu+1}(\kappa D)}{I_{\nu+1}(\kappa D)}.
\end{equation*}
Integrating $(r^{d-1}\Psi')'=\kappa^2r^{d-1}\Psi$ backward from $\Psi'(D)=0$ shows $\Psi'<0$ for $a<r<D$. Thus the range is $-\log\Psi(D)$. At fixed $a,D$, the large-argument Bessel asymptotics give
$\Psi(D)=2(a/D)^{(d-1)/2}e^{-\kappa(D-a)}(1+O(\kappa^{-1}))$,
which proves (ii).
\end{proof}

\subsection{Why We Use Reflection at Obstacles}
\label{app:proofs:reflection-choice}

Reflection lets the process remain in the free domain and continue toward the goal after reaching an obstacle boundary. Its corresponding condition is $\partial_n\Psi=0$, as in Eq.~\ref{eq:bc} and the reflecting walks of \citet{sawhney2023wost}. This choice also keeps $\log\Psi$, and so the control, finite up to the boundary under Assumption~\ref{ass:reg}.

\paragraph{Zero Dirichlet data produces a logarithmic singularity.}
Suppose instead that obstacles are absorbing, with zero terminal desirability. Denote the resulting solution by $\Psi_{\mathrm{abs}}$, retaining positive goal data and reflection at the outer workspace boundary. Then
\begin{equation*}
\Psi_{\mathrm{abs}}=0\text{ on }\partial\mathcal S,
\qquad
\log\Psi_{\mathrm{abs}}(x)\to-\infty,
\qquad
V_{\mathrm{abs}}(x)=-\lambda\log\Psi_{\mathrm{abs}}(x)\to+\infty
\end{equation*}
as $x$ approaches an obstacle from the free domain. Thus zero desirability represents an infinite collision penalty. The singularity is in the logarithmic representation and value function; the linear PDE solution $\Psi_{\mathrm{abs}}$ can remain smooth.

The control also becomes singular at a smooth obstacle surface. With the smoothness and separation assumptions of Assumption~\ref{ass:reg}, now replacing the obstacle condition by zero Dirichlet data, fix $z\in\partial\mathcal S$ and write $x=z-sn(z)$, where $s>0$ measures distance into the free domain. The Hopf boundary lemma gives $a_z=-\partial_n\Psi_{\mathrm{abs}}(z)>0$. Taylor expansion and $\mathbf u^\star=-\nabla V$, for $G=R=I$, yield
\begin{equation*}
\begin{split}
\Psi_{\mathrm{abs}}(z-sn(z))&=a_zs+O(s^2),\\
V_{\mathrm{abs}}(z-sn(z))&=-\lambda\log(a_zs)+O(s),\\
\mathbf u_{\mathrm{abs}}^\star(z-sn(z))&=-\frac{\lambda}{s}n(z)+O(1),
\qquad s\downarrow0.
\end{split}
\end{equation*}
The leading control points away from the obstacle and grows as inverse distance. A bounded log-output model cannot approximate this target uniformly all the way to the obstacle surface. Clipping the desirability to a positive floor removes the divergence numerically but introduces a boundary approximation.

\paragraph{Reflection preserves a bounded target.}
For the reflected problem, Proposition~\ref{prop:app-positivity} gives $\Psi_{\min}>0$ on the closed free domain. Consequently $V$ and $\log\Psi$ are bounded, $\partial_nV=0$, and
$\|\nabla V\|_\infty\leq\lambda\|\nabla\Psi\|_\infty/\Psi_{\min}<\infty$
under the stated regularity assumptions. The reflected dynamics enforce the boundary constraint; the Neumann condition does not impose a positive clearance from obstacles.

We therefore use reflection because it permits continued goal seeking with a bounded log target; other formulations need not use it. Strictly positive Dirichlet obstacle data $\Psi=e^{-\phi_{\mathrm{obs}}/\lambda}$ with finite $\phi_{\mathrm{obs}}$ avoids this logarithmic singularity, but models termination at an obstacle with a finite penalty. The singularity claim applies specifically to zero Dirichlet data for $\Psi$, not to Dirichlet conditions in general.

\subsection{The Chain Kernel}
\label{app:proofs:Q}

Throughout, $G^{B}$ and $G^{\kappa,B}$ denote the unscreened and screened Dirichlet Green's functions of $B=B(x,r)$ with pole at the center; both are radial in $\rho=\|\mathbf y-x\|$. By Lemma~\ref{lem:app-Qgen}, the step weight is the flux ratio
\begin{equation}
Q_d^\kappa(\rho,r)
=\frac{\partial_\rho G^{\kappa,B}}{\partial_\rho G^B}.
\label{eq:app-Q-flux}
\end{equation}
Both Green's functions are radial about the center, so the outward normal derivative at $\mathbf z$ factors as $\partial_nG=(\partial_\rho G)(n_{\mathbf z}\cdot\hat e_\rho)$ and the geometric factor cancels in the ratio, which is why Eq.~\ref{eq:app-Q-flux} is also the ratio of the two Poisson kernels in Eq.~\ref{eq:app-Pfactor}.

For $d\geq2$, $\kappa>0$ and $0<\rho\leq r$ the step weight is bounded on both sides,
\begin{equation}
\gamma_d(\kappa r)\;\leq\;Q_d^\kappa(\rho,r)\;<\;1 ,
\label{eq:weight-bounds}
\end{equation}
with $\gamma_d$ the full-sphere weight of Eq.~\ref{eq:Q-reduces}. By the upper bound, the weighted chain of Section~\ref{sec:method:step} remains bounded; the lower bound gives $Q_d^\kappa\geq\gamma_d(\kappa r_{\max})$ when the star radius is at most $r_{\max}$, which yields the replica bounds used in Appendix~\ref{app:proofs:delta}.
\begin{proof}[Bounds on $Q_d^\kappa$]
Positivity is immediate from Eq.~\ref{eq:app-Qgen}, since $K_{\nu+1}$, $I_{\nu+1}$ and $K_\nu/I_\nu$ are all positive. For the upper bound, set $D := G^{B} - G^{\kappa,B}$, which is radial and vanishes at $\rho = r$. Subtracting the two defining equations cancels the point sources and leaves $-\Delta D = \kappa^2 G^{\kappa,B}$ in $B$, so in flux form, the constant vanishing because the right side is locally integrable at the origin,
\begin{equation*}
-\rho^{d-1}\partial_\rho D = \kappa^2\!\!\int_0^\rho\! t^{d-1}G^{\kappa,B}(t)\,\mathrm{d}t ,
\qquad
-\rho^{d-1}\partial_\rho G^{B} = \frac{1}{\omega_d} .
\end{equation*}
Dividing the first by the second gives the weight-deficit identity
\begin{equation*}
1 - Q^{\kappa}_d(\rho,r) = \frac{\partial_\rho D}{\partial_\rho G^{B}}
= \kappa^2\!\!\int_{B(x,\rho)}\! G^{\kappa,B}(x,\mathbf{y})\,\mathrm{d}\mathbf{y},
\end{equation*}
The integrand is strictly positive on the interior of $B$, so the right side is strictly positive for $\rho>0$ and $Q^{\kappa}_d<1$. For the lower bound, $B(x,\rho)\subseteq B(x,r)$ and the integrand is nonnegative, so the deficit is at most $\kappa^2\!\int_B G^{\kappa,B}$; integrating $(-\Delta+\kappa^2)G^{\kappa,B}=\delta_x$ over $B$ gives $\kappa^2\!\int_B G^{\kappa,B} = 1-\int_{\partial B}P^{\kappa}_{B}\,\mathrm dS = 1-\gamma_d(\kappa r)$ by Proposition~\ref{prop:app-gamma}. Hence $\gamma_d(\kappa r)\leq Q^{\kappa}_d(\rho,r)<1$, with equality on the left at $\rho=r$. This follows directly from $B(x,\rho)\subseteq B(x,r)$; monotonicity in $\rho$ also follows from the same identity but is not needed.
\end{proof}

\subsection{Log-Space Variant}
\label{app:log-variant}

Alternatively, $\log\Psi_\theta$ can be fitted to log labels. There the weights are accumulated as $\Lambda^{(n)}=\sum_{j<\sigma^{(n)}\wedge K}\log W_j^{(n)}$ plus the log of the terminal value, which avoids underflow for long products, but $\log\widehat\Psi$ is a biased estimate of $\log\bar\Psi$, since the logarithm of a sample mean is not the mean of logarithms. The correction below removes that bias to $O(N^{-2})$.

\paragraph{Bias correction.}
For i.i.d. replicas with fixed bounds $0<z_{\min}\leq e^{\Lambda^{(n)}}\leq z_{\max}<\infty$ and variance $s^2$,
\begin{equation*}
\mathbb E[\log\widehat\Psi]
=\log\bar\Psi-\frac{s^2}{2N\bar\Psi^2}+O(N^{-2}).
\end{equation*}
The constants can depend on the replica bounds. For a fixed environment, such bounds follow from a finite horizon and a bounded target log field (Appendix~\ref{app:proofs:delta}).

For $N\geq2$, compute all ratios from shifted exponentials. Set $b=\max_n\Lambda^{(n)}$, $z_n=e^{\Lambda^{(n)}-b}$, $\bar z=N^{-1}\sum_nz_n$, and $s_z^2=(N-1)^{-1}\sum_n(z_n-\bar z)^2$. Then
\begin{equation}
\widehat v=\frac{s_z^2}{N\bar z^2},\qquad
\widehat\Lambda=\log\widehat\Psi+\frac{\widehat v}{2}
=b+\log\bar z+\frac{\widehat v}{2}.
\label{eq:label}
\end{equation}
Here $s_z^2$ is the sample variance of the shifted replicas $z_n$. The random common shift makes those replicas dependent; unbiasedness is invoked only for the unshifted replicas below. Writing $\widehat s^2$ for the unbiased sample variance of the unshifted replicas $e^{\Lambda^{(n)}}$, the ratio is unchanged by the common shift, $\widehat v=s_z^2/(N\bar z^2)=\widehat s^2/(N\widehat\Psi^2)$, so Eq.~\ref{eq:label} adds the bias correction $\widehat s^2/(2N\widehat\Psi^2)$. The correction removes the leading log bias, leaving $O(N^{-2})$ under the stated assumptions.

\subsection{Delta-Method Bias of the Log Label}
\label{app:proofs:delta}

\begin{proposition}[Bias of a log sample mean]
\label{prop:app-delta}
Let $Z_1,\ldots,Z_N$ be i.i.d. with $Z_i\in[z_{\min},z_{\max}]$, $0<z_{\min}\leq z_{\max}<\infty$, mean $\mu$, and variance $s^2$. For $N\geq2$, write $\widehat\Psi=N^{-1}\sum_iZ_i$ and let $\widehat s^2$ be the unbiased sample variance. Then
\begin{equation*}
\mathbb E[\log\widehat\Psi]
=\log\mu-\frac{s^2}{2N\mu^2}+O(N^{-2}),\qquad
\mathbb E\!\left[\log\widehat\Psi+\frac{\widehat s^2}{2N\widehat\Psi^2}\right]
=\log\mu+O(N^{-2}).
\end{equation*}
The constants depend on the replica bounds.
\end{proposition}
\begin{proof}
Set $U=\widehat\Psi-\mu$. Taylor's theorem through degree three gives
$\log(\mu+U)=\log\mu+U/\mu-U^2/(2\mu^2)+U^3/(3\mu^3)+R_4$ with $|R_4|\leq C U^4$ because $\widehat\Psi\geq z_{\min}$. Independence gives $\mathbb E U=0$, $\mathbb E U^2=s^2/N$, $\mathbb E U^3=O(N^{-2})$, and $\mathbb E U^4=O(N^{-2})$, proving the first expansion.

For the correction, expand $(\mu+U)^{-2}=\mu^{-2}-2\mu^{-3}U+R_2$, where $|R_2|\leq C U^2$. Since $\widehat s^2$ is bounded, $\mathbb E[\widehat s^2R_2]=O(N^{-1})$. Also $\mathbb E\widehat s^2=s^2$ and $\mathbb E[\widehat s^2U]=m_3/N$, with $m_3$ the third central moment of $Z_i$. Hence $\mathbb E[\widehat s^2/\widehat\Psi^2]=s^2/\mu^2+O(N^{-1})$, which proves the corrected expansion. Jensen's inequality gives the nonpositive bias of the uncorrected logarithm for every $N$.
\end{proof}
For a fixed environment, horizon $K$, and target bounded by $0<\Psi^-_{\min}\leq\Psi_{\theta^-}\leq\Psi^-_{\max}<\infty$, let $W_{\min}>0$ be a lower bound on the step weight, which exists by Eq.~\ref{eq:weight-bounds} when the star radius is bounded. The walk replicas then lie in $[z_{\min},z_{\max}]$ with
$z_{\min}=W_{\min}^K\min(g_{\min},\Psi^-_{\min})$ and $z_{\max}=\max(1,\Psi^-_{\max})$.
The target bounds hold because the log field is continuous on the compact domain. They are not assumed uniform over all trained parameters or environments. The shifted calculation in Eq.~\ref{eq:label} preserves the variance ratio exactly in real arithmetic and avoids underflow from a common small scale.

\subsection{Bounded Local Search}
\label{app:proofs:grid}

Recall the lattice of Section~\ref{sec:method:local-search}, with vertices $\mathcal X_h$ and edges between axis-adjacent free cells. We take this graph to be finite and $\Psi_\theta$ finite at every vertex.

\paragraph{Reachable endpoint scores.}
The largest desirability reachable in $k$ moves is computed by
\begin{equation}
B_0(x)=\Psi_\theta(x),\qquad
B_{k+1}(x)=\max_{y\in\{x\}\cup\mathcal N(x)}B_k(y).
\label{eq:local-propagation}
\end{equation}
Indeed, $\mathcal R_0(x)=\{x\}$ and
\begin{equation*}
\mathcal R_{k+1}(x)=\mathcal R_k(x)\cup\bigcup_{y\in\mathcal N(x)}\mathcal R_k(y),
\end{equation*}
so induction gives $B_k(x)=\max_{z\in\mathcal R_k(x)}\Psi_\theta(z)$. Equation~\ref{eq:local-move} therefore compares $B_{L-1}$ at candidate neighbors. Moving to that neighbor uses the first of the $L$ moves.

\paragraph{Recovering a route.}
In each propagation layer, the previous score is kept unless a neighbor has a strictly higher score, in which case the corresponding edge is stored. Starting at horizon $L$, follow the stored edges while decreasing the remaining horizon. A layer with no edge is skipped because its maximum was already reachable at a shorter horizon. Induction on the remaining horizon gives a feasible route of at most $L$ edges ending at a cell with score $B_L(x_0)$. Executing a prefix preserves graph feasibility, although the next replan may select a different endpoint. Continuous-motion feasibility additionally requires valid edges; collision-free sampled endpoints are insufficient. The implemented lattice uses cells marked free in the saved occupancy mask, connected through face-adjacent free cells. Continuous segments from off-center states are checked across all traversed voxels. All edges are checked for collisions. For $e=1$, Eq.~\ref{eq:local-move} instead selects a neighbor anew from the fixed $B_{L-1}$ map at each visited state.

\paragraph{Saturation at large depth.}
The reachable sets are nested, hence $B_{k+1}(x)\geq B_k(x)$. On a connected component $C$ of graph diameter $D$, every $k\geq D$ gives $B_k(x)=\max_{z\in C}\Psi_\theta(z)$ for all $x\in C$. Once $L-1\geq D$, the candidate scores in Eq.~\ref{eq:local-move} coincide and direction depends on tie-breaking. For a reconstructed route, maximizing the terminal score likewise does not distinguish the lengths of paths to the same endpoint. Thus increasing depth enlarges the set of candidate endpoints, but does not guarantee a shortest path or a successful rollout.

\section{Supplementary Experiments}
\label{app:experiments}

\subsection{Evaluation and Training Protocol}
\label{app:exp-protocol}

\paragraph{Evaluation scope.} The pretrained 2D application benchmark contains 592 maps (37,888 starts); the 2D generated and 4D manipulator test sets each contain 512 maps (32,768 starts). Every map uses eight goals and every goal uses eight starts, chosen at the centers of eight equal goal-distance quantile bins. Success is evaluated on the matched saved queries, including failed rollouts in the denominator. Field prediction or solve, controller preprocessing, and rollout/search all contribute to reported test-time cost. These costs are amortized over the corresponding starts; they include more than the neural forward time. FEM solutions provide the numerical ground-truth fields for reconstruction-error evaluation.

\paragraph{Execution model for replanning comparisons.}
\label{app:replanning-protocol}
Table~\ref{tab:replanning} compares nominal planners, sampling-based controllers, and PUN at execution-noise levels $\lambda\in\{0.25,1,2,4\}$.
Before boundary handling, the proposed transition is $\widetilde x_{k+1}=x_k+a_k+\sqrt{\lambda\Delta t}\,z_k$, with independent $z_k\sim\mathcal N(0,I_d)$, $\Delta t=1/65{,}536$, and $\|a_k\|_2\leq 1/(8n)$, where $n$ is the largest grid dimension. Each control interval contains eight microsteps. Coordinates occupy a unit square in 2D, a $1\times1\times\tfrac12$ domain in 3D, and normalized joint coordinates in 4D. The noise-increment covariance is $\lambda\Delta t I_d$, consistent with the PDE specialization $\Sigma=\lambda I_d$.
In these replanning evaluations, all compared methods use the same execution and collision-checking rules. Collision checks cover the entire executed motion segment, and an obstacle hit is counted as a rollout failure. Thus, success requires reaching the goal without an obstacle hit or an outer-domain boundary violation. The reflected PDE provides the planning field, whose effectiveness is evaluated here under collision-failure execution.
Joint limits are not modeled in these experiments.
Moves outside the outer workspace or configuration-domain boundary are also counted as rollout failures.
For each environment--goal query at a given noise level, PUN predicts the entire field in a single forward pass and reuses it throughout the rollout. The reported replanning comparisons use positive noise levels, consistent with the $\lambda>0$ scope of the screened-Poisson derivation.
Success requires entry into the saved goal-cell mask before failure. Methods replan from the current noisy state every eight microsteps. Execution also stops on a collision failure, an unavailable route, or the controller's stopping condition.

\paragraph{Baseline configuration.}
The MPPI and PI$^2$-CMA implementations share a horizon of $H=32$ steps, $K=128$ sampled trajectories per optimization iteration, and one iteration per planning call. Both use double precision, componentwise action bounds $[-0.05,0.05]$, initial exploration standard deviation $0.025$, and optimization temperature $\eta=0.5$. This temperature is distinct from the execution-noise parameter $\lambda$. A fresh call initializes every action to $(x_{\mathcal G}-x)/H$, where $x_{\mathcal G}$ is the goal center. Thus each call evaluates $KH=4{,}096$ sampled transitions, followed by a rollout of the selected nominal sequence. The wrapper supports resuming interrupted optimization from saved parameters and random-generator state; ordinary calls initialize a new optimizer rather than carrying its solution between replans.
The configuration uses $(H,K,I)=(32,128,1)$, where $I$ is the iteration count. These factors describe nominal sampled work $IKH$. No validation-based parameter sweep was performed. We use standard values for temperature, exploration scale, action bounds, the objective weights.

A$^\star$ and Dijkstra use standard C++ implementations with face-neighbor connectivity and unit edge costs. Grids are the native 2D maps, $64\times64\times32$ in 3D, and $17^4$ in 4D. A$^\star$ uses Manhattan distance to the goal-mask bounding box; Dijkstra uses zero heuristic. CPU RRT uses OMPL 2.0.1, goal bias 0.05, uniform domain sampling, and maximum extension 0.4 in normalized coordinates. Search stops on a goal solution; there is no fixed sample or tree-size cap.

\paragraph{Planning objective and geometry.}
Both sampling-based optimizers use the following task cost, with stage costs evaluated after each predicted transition:
\begin{align*}
\ell_t &= \mathbf{1}_{\{x_{t+1}\notin\mathcal G\}}
\left(1+10\|x_{t+1}-x_{\mathcal G}\|_2^2+\|u_t\|_2^2\right)+100 I_t,\\
\phi_H &= 100\,\mathbf{1}_{\{x_H\notin\mathcal G\}}\|x_H-x_{\mathcal G}\|_2^2.
\end{align*}
Here $I_t$ flags a proposed segment rejected by the geometry-validity routine when its starting state is not already in the goal. Before calling the geometry step routine, the wrapper replaces rejected actions and actions starting in the goal with zero displacement. This is a penalized planning surrogate; executed collisions are failures under the evaluation protocol above. Goal absorption is delegated to the geometry step routine. After optimization, the wrapper returns the predicted state path, truncated at its first recorded goal state. Execution noise is applied outside the optimization wrapper; its action-exploration distribution should not be identified with the execution-noise distribution.
Geometry routines traverse all crossed voxels, including incident cells at edge/corner crossings, to detect goal entry, obstacle contact, and domain exit. Collision takes precedence over simultaneous goal entry in first-contact scoring. Actions are displacements in the normalized coordinates specified above. Predicted trajectories use deterministic geometry; the outer rollout adds execution noise. Returned paths are subdivided into bounded microsteps.

\paragraph{MPPI implementation.}
We use the \texttt{MPPI} class from \texttt{pytorch\_mppi}\footnote{\url{https://github.com/UM-ARM-Lab/pytorch_mppi}}~\citep{williams2017mppi}, with $\texttt{noise\_sigma}=0.025^2 I_d$, $\texttt{lambda\_}=0.5$, the task costs above, and the stated action bounds. The one call to \texttt{command} refines the nominal sequence at the same state with \texttt{shift\_nominal\_trajectory=False}. We retain the library's MPPI update and control-perturbation cost in addition to the supplied task cost, and generate the returned path from the final nominal action sequence. The implementation uses the original MPPI update and control-perturbation cost.

\paragraph{PI$^2$-CMA implementation.}
Our adaptation of PI$^2$-CMA~\citep{stulp2012pi2cma} represents the action sequence by 16 uniformly spaced knots per coordinate, linearly interpolated to the 32-step horizon with aligned endpoints and clipped to the action bounds. A single parameter vector is sampled for each trajectory from $\mathcal N(\theta,\Sigma)$, with initial covariance $\Sigma=0.025^2 I_{16d}$. For sampled parameters $\theta_k$, let $S_{kt}$ be the remaining task cost from step $t$, including the terminal cost once. We compute
\begin{equation*}
P_{kt}=\frac{\exp[-(S_{kt}-\min_j S_{jt})/\eta]}{\sum_j\exp[-(S_{jt}-\min_l S_{lt})/\eta]},
\qquad
w_k=\sum_{t=0}^{H-1}\frac{H-t}{H(H+1)/2}\,P_{kt},
\end{equation*}
then update $\theta^+=\sum_k w_k\theta_k$ and $\Sigma^+=\sum_k w_k(\theta_k-\theta)(\theta_k-\theta)^\top$, centering the covariance at the old mean. This implements temporal averaging of the per-step updates for exploration that is constant over a trajectory. The implementation symmetrizes the covariance and samples using an eigendecomposition, checking negative eigenvalues against a numerical tolerance before clipping them to zero; it imposes no positive exploration floor. We use a fixed temperature on unnormalized costs, rather than the cost-range normalization used in the original paper's experiments.
The custom wrapper implements knot-based PI$^2$-CMA, without CMA-ES evolution paths, rank-based selection, or step-size adaptation.

\paragraph{Hardware and timing.}
Both CPU and CUDA implementations were evaluated; the CPU version was faster in the implementation comparison. The optimizer's wall-clock timer includes initialization, the one optimization iteration, logging, the final nominal rollout, and conversion of the returned path to host memory. The CUDA implementation also synchronizes after each iteration and before returning. These are per-call measurements; the tables report per-start costs and require the outer execution loop's accounting across planning calls.
Table~\ref{tab:replanning} uses CPU baseline timings in 2D/3D/4D since they are faster than GPU, and GPU PUN timings throughout. Hardware is an Intel Xeon Platinum 8480+ or an NVIDIA B200. CPU software is Python 3.11.11, PyTorch 2.7.1+cpu, and NumPy 1.26.4; GPU software is Python 3.9.21, PyTorch 2.7.1+cu128, and NumPy 2.0.2.

Timing excludes loading and warmup and uses three observations per selected pair, method, and noise level for the benchmark sets. Variability is minimal. Per-start measurements include initialization, all replanning calls, geometry, and execution, including failed rollouts; the number of calls depends on the realized trajectory. GPU work is synchronized at timing boundaries.

PUN timings include fresh field prediction and preparation for each start, with the field reused during replanning.

\paragraph{Statistical reporting for replanning.}
Table~\ref{tab:replanning} uses 500 maps in 2D, 250 in 3D, and 480 in 4D, with three goals per map and one start per goal. Each query uses one seeded execution-noise combination at each noise level; methods share the saved queries and Gaussian noise sequence. Rollouts receive equal weight, which also gives maps equal weight because each contributes three queries. One trained model per dimension is evaluated. We compute 95\% intervals using 50,000 paired map-level bootstrap resamples, retaining all queries within each sampled map. Nondegenerate intervals are the estimate $\pm1.96$ bootstrap standard errors. Runtime intervals use separate samples of maps, respectively.

\begin{table}[!t]
\centering
\caption{95\% confidence intervals for replanning success rates (\%) corresponding to Table~\ref{tab:replanning}.}
\label{tab:replanning-ci}
\label{tab:replanning-success-ci}
{\footnotesize
\setlength{\tabcolsep}{4pt}
\begin{tabular*}{\linewidth}{@{\extracolsep{\fill}}llrrrr@{}}
\hline
Dim. & Method & $\lambda=0.25$ & $\lambda=1$ & $\lambda=2$ & $\lambda=4$ \\
\hline
\multirow{6}{*}{\textbf{2D}} & A* & [55.92, 62.08] & [39.16, 45.50] & [31.90, 38.10] & [20.89, 26.45] \\
 & Dijkstra & [55.81, 62.19] & [39.15, 45.51] & [31.99, 38.01] & [21.08, 26.26] \\
 & RRT & [79.37, 84.63] & [65.68, 71.66] & [49.14, 55.52] & [33.46, 39.88] \\
 & MPPI & [81.70, 86.30] & [57.77, 64.23] & [42.45, 48.89] & [28.82, 34.52] \\
 & PI$^2$-CMA & [75.76, 80.90] & [56.81, 63.19] & [44.76, 51.24] & [31.95, 38.05] \\
 & PUN (zero-shot) & [89.90, 93.44] & [71.27, 76.73] & [59.15, 65.51] & [42.55, 48.79] \\
\hline
\multirow{6}{*}{\textbf{3D}} & A* & [32.37, 40.97] & [19.38, 27.28] & [9.10, 14.90] & [4.04, 7.96] \\
 & Dijkstra & [33.00, 41.66] & [18.98, 26.36] & [9.05, 14.95] & [3.73, 8.27] \\
 & RRT & [48.85, 57.81] & [30.86, 39.80] & [21.28, 29.38] & [8.69, 15.31] \\
 & MPPI & [58.44, 66.90] & [30.43, 38.91] & [14.98, 22.36] & [4.93, 9.73] \\
 & PI$^2$-CMA & [42.92, 51.74] & [18.36, 25.64] & [17.51, 25.15] & [6.07, 11.27] \\
 & PUN (zero-shot) & [89.60, 94.40] & [64.44, 72.90] & [35.62, 44.38] & [12.71, 19.29] \\
\hline
\multirow{6}{*}{\textbf{4D}} & A* & [100.00, 100.00] & [95.39, 97.95] & [69.54, 75.46] & [55.94, 62.40] \\
 & Dijkstra & [100.00, 100.00] & [95.68, 97.66] & [69.66, 75.34] & [56.06, 62.28] \\
 & RRT & [97.53, 99.13] & [94.65, 97.01] & [87.97, 92.03] & [83.82, 87.84] \\
 & MPPI & [93.68, 96.32] & [90.59, 94.41] & [79.15, 84.19] & [68.91, 74.43] \\
 & PI$^2$-CMA & [95.43, 97.91] & [90.70, 94.30] & [84.40, 88.94] & [75.77, 80.89] \\
 & PUN & [100.00, 100.00] & [100.00, 100.00] & [100.00, 100.00] & [93.49, 96.51] \\
\hline
\end{tabular*}
}
\end{table}

\begin{table}[!t]
\centering
\caption{95\% confidence intervals for replanning runtime (ms/start) corresponding to Table~\ref{tab:replanning}.}
\label{tab:replanning-runtime-ci}
{\footnotesize
\setlength{\tabcolsep}{4pt}
\begin{tabular*}{\linewidth}{@{\extracolsep{\fill}}llrrrr@{}}
\hline
Dim. & Method & $\lambda=0.25$ & $\lambda=1$ & $\lambda=2$ & $\lambda=4$ \\
\hline
\multirow{6}{*}{\textbf{2D}} & A* & [24.99, 62.27] & [26.25, 67.93] & [28.10, 74.26] & [29.44, 78.08] \\
 & Dijkstra & [79.37, 230.51] & [88.65, 261.25] & [102.34, 304.48] & [123.40, 356.54] \\
 & RRT & [69.41, 127.49] & [72.10, 135.64] & [78.18, 146.36] & [82.46, 166.96] \\
 & MPPI & [1227.92, 2240.30] & [1257.88, 2288.74] & [1279.80, 2397.66] & [1400.54, 2851.70] \\
 & PI$^2$-CMA & [964.06, 1697.46] & [1053.67, 1916.15] & [1067.48, 1981.92] & [1079.98, 2133.66] \\
 & PUN (zero-shot) & [7.79, 11.31] & [8.80, 13.48] & [10.71, 17.43] & [10.19, 18.83] \\
\hline
\multirow{6}{*}{\textbf{3D}} & A* & [50.74, 96.06] & [34.34, 86.60] & [42.19, 105.51] & [40.32, 90.18] \\
 & Dijkstra & [50.08, 96.84] & [35.29, 90.79] & [43.41, 110.57] & [41.51, 93.95] \\
 & RRT & [155.27, 321.35] & [93.18, 246.62] & [76.54, 169.16] & [94.56, 157.90] \\
 & MPPI & [382.10, 682.52] & [408.91, 746.77] & [433.13, 675.09] & [526.69, 881.93] \\
 & PI$^2$-CMA & [332.42, 572.08] & [305.61, 564.41] & [371.44, 580.52] & [482.06, 731.04] \\
 & PUN (zero-shot) & [14.95, 16.33] & [14.61, 16.07] & [13.26, 15.34] & [13.69, 15.17] \\
\hline
\multirow{6}{*}{\textbf{4D}} & A* & [12.32, 21.30] & [12.25, 21.91] & [35.26, 120.78] & [28.18, 96.70] \\
 & Dijkstra & [12.51, 21.83] & [12.37, 22.31] & [31.21, 102.29] & [27.73, 91.05] \\
 & RRT & [16.26, 32.14] & [16.20, 35.40] & [23.09, 62.91] & [28.09, 48.53] \\
 & MPPI & [332.13, 490.35] & [329.20, 467.16] & [297.34, 431.10] & [300.44, 449.22] \\
 & PI$^2$-CMA & [300.61, 458.97] & [291.21, 424.17] & [284.19, 426.69] & [288.82, 424.62] \\
 & PUN & [10.50, 10.76] & [10.32, 10.70] & [10.37, 10.73] & [11.43, 11.89] \\
\hline
\end{tabular*}
}
\end{table}

\paragraph{Implementation of the WoSt estimator.}
\label{app:wost-implementation}
The experiments use the following estimator and numerical choices.
\begin{enumerate}
\item \emph{Absorption and radius floor.} A walk reaches $\Gamma_D$ only in the limit, so absorption is triggered on the shell $\{d_D<\varepsilon\}$ and the terminal value is $g$ at the nearest point of $\Gamma_D$; $\varepsilon$ is chosen below the tubular-neighborhood radius of $\Gamma_D$ and its distance to $\Gamma_N$. The star radius also carries a floor, $r(x)=\min\{d_D(x),\max(r_{\min},d_{\mathrm{sil}}(x))\}$, so that steps do not become arbitrarily short near a silhouette. The shell costs at most $C_\Psi\varepsilon$ (Eq.~\ref{eq:app-shell-error}); the floor can break star-shapedness, and the error it introduces has unknown sign~\citep{sawhney2023wost}.
\item \emph{Weighted survival.} The solid-angle sampling law in Eq.~\ref{eq:kernel} leaves the scalar screening factor $Q_d^\kappa$ as the per-step weight. We multiply these factors rather than sample killing events or apply Russian roulette. Conditioning out the killing randomness removes that source of variance and avoids zero replicas caused by stochastic killing.
\item \emph{Fixed-horizon continuation.} Walks stop at the prescribed step cap or the goal. Capped paths contribute their accumulated weight times the detached EMA prediction at the endpoint, as in Section~\ref{sec:method:labels}. Goal-reaching paths use the known boundary value. Every requested trajectory remains in the sample mean, including capped paths.
\item \emph{Log accumulation and bias correction.} The trace records the sum of log step weights independently of their floating-point product, so product underflow does not erase the log evidence. The log-label variants use log-sum-exp and the sample-variance correction in Eq.~\ref{eq:label}. The direct-label variants use the weighted arithmetic mean without this correction. Direct and log targets are constructed from the same sampled paths.
\item \emph{Reuse of model-independent walks.} Cached endpoints, weights, and completion states are independent of the neural parameters. The EMA continuation is reevaluated during training rather than frozen into the cache.
\item \emph{Stable screening kernels.} The voxel backend for dimensions three through seven evaluates the full-sphere factor using the positive series
\begin{equation*}
\gamma_d(z)=\left[\sum_{j=0}^{\infty}a_j\right]^{-1},
\qquad a_0=1,\qquad
 a_j=a_{j-1}\frac{z^2}{4j(\nu+j)},\qquad \nu=\frac d2-1.
\end{equation*}
The CUDA implementation sums in double precision and uses a logarithmic large-argument asymptotic branch. For interior reflecting hits, it evaluates Eq.~\ref{eq:app-Qgen} with $e^{-z}I_\nu(z)$ and $e^zK_\nu(z)$, combining nonnegative terms with an exponential factor no greater than one. With $t=\kappa\rho$, $T=\kappa r$, and $\delta=T-t$, the near-sphere branch uses
\begin{equation*}
Q_d^\kappa(\rho,r)\approx\gamma_d(T)
\left[1+\frac{\delta^2}{2}-\frac{d-1}{6T}\delta^3\right].
\end{equation*}
This reproduces the sphere endpoint without two separately rounded Bessel evaluations. The CUDA branches use this expansion for $\delta\leq2\cdot10^{-6}(1+T)$ and the limiting value one for $t\leq10^{-7}$ or $T=0$.
\item \emph{Voxel boundary queries.} The same backend uses bounding-volume hierarchies for goal faces and view-dependent reflecting silhouettes, and voxel traversal for first boundary hits. This avoids triangulating the voxel boundary. Grid-plane and corner handling preserves the appropriate incident free cell and resolves boundary hits before accepting a sphere endpoint. These geometric operations implement the first-hit chain in Eq.~\ref{eq:chain}.
\end{enumerate}

\paragraph{Shell error.}
Let $\sigma_\varepsilon=\inf\{k:d_D(\xi_k)<\varepsilon\}$ and $\tau=\sigma_\varepsilon\wedge K$ for finite $K$. Assume admissible active steps, the smooth-domain regularity assumptions, and an $\varepsilon$ below the goal's tubular-neighborhood radius and its separation from reflecting boundaries. Define $\Psi_{\varepsilon,K}(x)$ as the expected product of weights through $\tau$ times $g(\pi_D\xi_\tau)$ if $\sigma_\varepsilon\leq K$, or the true continuation $\Psi(\xi_K)$ otherwise. Then
\begin{equation}
\|\Psi_{\varepsilon,K}-\Psi\|_\infty\leq C_\Psi\varepsilon,
\qquad C_\Psi=\sup_{\overline\Omega}\|\nabla\Psi\| .
\label{eq:app-shell-error}
\end{equation}
To prove this, bounded stopping gives $\Psi(x)=\mathbb E_x[(\prod_{j<\tau}W_j)\Psi(\xi_\tau)]$. Only the shell event changes the terminal value. Since $g(\pi_D\xi_\tau)=\Psi(\pi_D\xi_\tau)$, the projection segment lies in the closed free domain, and its length is below $\varepsilon$, the terminal difference is bounded by $C_\Psi\varepsilon$. The product of weights is at most one. This proves the bound without an infinite-walk termination assumption. If $\sigma_\varepsilon<\infty$ almost surely for each start, bounded convergence also yields the same bound for the uncapped shell estimator. Replacing true continuation by a fixed bounded target adds its discounted endpoint error on $\{\sigma_\varepsilon>K\}$, as described in Section~\ref{sec:method:labels}. The radius floor is not covered: a clamped step can violate admissibility and introduce an error of unknown sign~\citep{sawhney2023wost}.

Antithetic directions, first-step gradient control variates, weight-floor stopping, expected-clock stopping, and randomized horizons are disabled in the reported pretrained models. In particular, the full-sphere gradient control variate is not transferred to arbitrary WoSt stars.

\paragraph{Backbone compute policy.} FNO uses 200 epochs. Each alternative backbone has a maximum of 200 epochs and a registered training-compute allowance of 66,856.484 seconds; an epoch is not started if it is not expected to fit within the remaining allowance. Each completed epoch uses 64 training maps and eight goals per map. The 200-epoch cosine schedule is retained when a run stops at the compute cap. The allowance was derived from 1.5 times an estimated individual FNO training cost; shared work in the reference run means it is not an exact equal-GPU-time comparison. Completed training epochs were 200 for FNO, 200 for U-NO, 185 for U-Net, 13 for GINO, 200 for LocalNO, 200 for DISCO, 146 for UPT, and 134 for Transolver. Our overall method is architecture-agnostic; we chose FNO because it gives the highest planning success among the tested implementations under this compute policy, but our overall method is compatible with other neural-operator architectures. Planning benefits from balancing sensitivity to local obstacle geometry with a global view of the environment.

\paragraph{Supervision parameterization.} The primary WoSt rows fit direct desirability with value, goal, and gradient loss terms. The loss ablations use 200 epochs and remove one term at a time. 
 All predictions are evaluated through the stated controller; a small field error does not guarantee path success.

\paragraph{Local-search implementation.} Unless a depth is stated, Local search denotes Local search (10). The finite-lattice controllers use legal axis-adjacent moves, with cycle detection and a step cap. \emph{Greedy} uses $(L,e)=(1,1)$ and requires strict improvement. \emph{Local search} (4) uses $(L,e)=(5,1)$ in 2D, 3D, and 4D. \emph{Local search} (10) uses $(L,e)=(10,3)$ with route reconstruction in 2D, and $(L,e)=(11,1)$ in 3D and 4D. In the single-move variants, four or ten propagation steps are counted after the candidate first move, explaining the extra edge in $L$. Direct models use their raw finite desirability predictions, without clipping or enforcing a goal score during search. For the convention $\ell_\theta=\log\Psi_\theta$ in Appendix~\ref{app:log-variant}, the maximizing search score is $\ell_\theta$. If a model instead predicts $-\log\Psi_\theta$ or $V_\theta$, its output must be negated to preserve this ranking (for $\lambda>0$ in the latter case). The implemented log-label model predicts $\ell_\theta=-\operatorname{softplus}(z_\theta)$, reconstructs desirability as $\exp(\ell_\theta)$, and uses $-z_\theta$ as its search score. Maximizing $-z_\theta$ gives the same ranking as maximizing $\ell_\theta$, consistent with this convention. Success requires entering the saved goal mask. Ties in the single-move rule favor the candidate's own field score, then fewer previous visits, then a fixed lexicographic coordinate order (ascending in 2D, descending in 3D/4D). Route reconstruction retains an existing choice on equal scores, using a fixed axis traversal order for competing improvements. The 2D ten-step controller stops when the best endpoint does not improve the current score by more than $10^{-12}$. Missing legal moves, a fifth visit to a cell under local search, or exhaustion of the rollout budget are recorded as failures. The local diagrams below use one stored field to display these rules, with no new benchmark outcomes inferred from the illustration.

\noindent Details of the walk-horizon study and qualitative resolution-transfer examples are provided below.
\
\subsection{Walk Horizon and Label Budget}
\label{sec:exp:labels}

\paragraph{Label budget and planning success.} The controlled horizon study in Table~\ref{tab:walk-cap-main} uses 16 training maps, eight goals per map, four trajectories per query, and 40 epochs, with evaluation on the generated 2D test set. With the intermediate 100- and 400-step settings, gradient / local-search (10) success is 40.55\% / 66.80\% and 38.06\% / 69.69\%, respectively. Local-search success is highest at 400 steps (69.69\%) and 0.03 percentage points lower at 800 steps, whereas gradient success is highest at 100 steps. A separate study of budget allocation without pretraining is reported in Appendix~\ref{app:offline-experiments}.

\FloatBarrier

\begin{table}[!htbp]
\centering
\caption{FNO (WoSt) planning success (\%) with gradient following and local search.}
\label{tab:controller-dimensions}
\begin{tabular}{lrr}
\hline
Test set & Gradient & Local search (10) \\
\hline
3D application benchmarks & 65.95 & 89.46 \\
4D manipulators & 97.50 & 99.48 \\
\hline
\end{tabular}
\end{table}

\begin{table}[!htbp]
\centering
\caption{Full test-time cost (ms/start) on the 2D zero-shot application-inference subset and 4D manipulator test set. Both methods use local search in 2D and in 4D.}
\label{tab:infer-compact}
\label{tab:local-search-supplement}
\begin{tabular}{lrr}
\hline
Method & 2D zero-shot & 4D manipulators \\
\hline
WoSt + FNO & 1.20 & 2.79 \\
Direct WoSt & 1128.60 & 160.59 \\
\hline
\end{tabular}
\end{table}

\subsection{Resolution Generalization}
\label{sec:exp:resolution}

Figure~\ref{fig:resolution-main} shows FNO predictions on held-out maps at $128\times128$, $200\times200$, and $256\times256$. Each field supports paths from multiple starts to one goal. These examples illustrate field reuse at different discretizations; the aggregate comparison across architectures and larger grids is Fig.~\ref{fig:main-collage}(b).

\begin{figure}[!htb]
\centering
\includegraphics[width=1.0\linewidth]{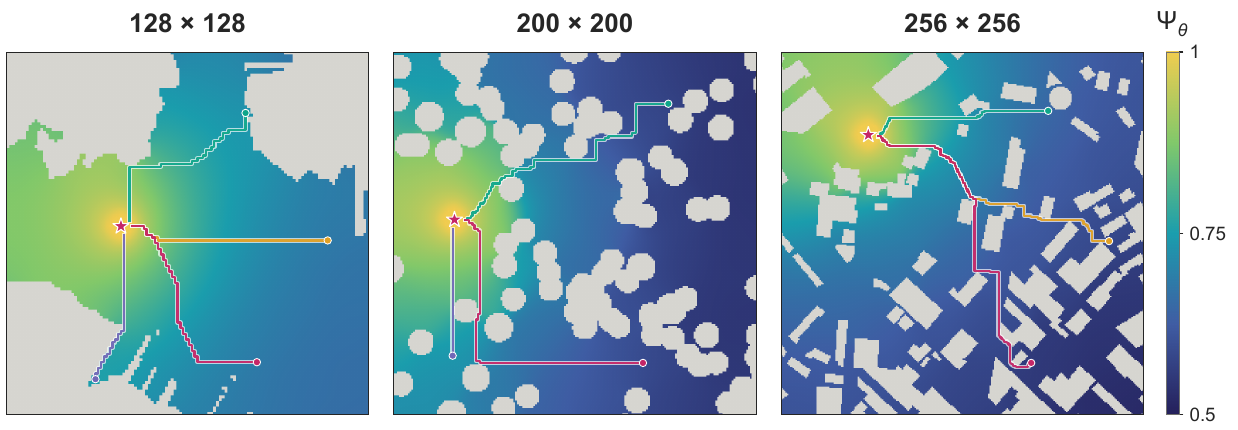}
\caption{Zero-shot FNO field predictions and extracted paths at $128\times128$, $200\times200$, and $256\times256$ (left to right). Color shows predicted desirability $\Psi_\theta$, gray denotes obstacles, circles mark starts, and the star marks the shared goal in each map.}
\label{fig:resolution-main}
\end{figure}

\section{Benchmarks}
\label{app:benchmark-names}

\paragraph{2D (zero-shot; 592 maps).}
AI4Arctic Sea Ice (Arctic sea-ice maps from satellite observations); Fields2Benchmark (agricultural field geometries); HiRISE DTM (terrain models of the surface of Mars); HydroLAKES ASV (lake-shoreline geometries for autonomous surface vessels); LOLA LDEM (lunar elevation maps from laser altimetry); MovingAI BG512 (Baldur's Gate maps); MovingAI Street (urban street-map grids); NOAA ENC Harbor (harbor geometries from navigational charts); Open Forest Observatory (forest structure maps from drone imagery); PNO2D (synthetic obstacle environments); TartanDrive 2 (terrain from off-road driving data).

\paragraph{3D (zero-shot; 278 maps).}
Dryad SIC Foam (porous foam geometries); Falkor MAR AUV (Mid-Atlantic Ridge seafloor for autonomous underwater vehicles); Flybo UAV (scenes for aerial-robot navigation); HVaCT Liver Tip (liver and vessel anatomy for navigation); Industrial Pipe Routing (industrial layouts for pipe routing); Lung Needle (lung anatomy for needle-path planning); MOSAiC Under-Ice AUV (Arctic under-ice geometry for autonomous underwater vehicles); PNO3D iGibson/HouseExpo (indoor scene and floor-plan geometries); ReplicaCAD Indoor UAV (indoor CAD scenes for aerial navigation); Trabecular Bone (porous internal bone structure); Underwater AT40-02 MAR (seafloor surveyed for the Mid-Atlantic Ridge); Axial North Rift (seafloor terrain at Axial Seamount's north rift); Escanaba (seafloor terrain in Escanaba Trough); Loihi Pele's Pit (a submarine volcanic crater); Squid5 EDR (Eastern Dry Rocks coral reef topography from underwater imagery); USGS Tahoe (Lake Tahoe lakebed topography).

\paragraph{4D (512 maps).}
eCobra600 (Omron industrial robot arm); Kinova MICO (robot arm); OpenManipulator (ROBOTIS open-source robot arm); qArm (Quanser robot arm). Each 4D state represents a manipulator configuration in joint coordinates.

\section{Additional Results}
\label{app:additional}

\subsection{Label-Budget Allocation without Pretraining}
\label{app:offline-experiments}

\paragraph{Allocation of a fixed label budget.} Shifting from 64 walks per query and 8 goals per map to 4 walks and 128 goals reduces 2D test success from 62.6\% to 51.5\% for DISCO and from 61.6\% to 56.3\% for FNO. In 3D, the same shift raises FNO success from 25.0\% to 28.8\%; Transolver success changes from 6.3\% to 7.0\%, while label-generation cost rises from 244.35 to 621.46 GPU hours. Equal nominal label budgets therefore need not imply equal computation or a common preferred allocation.

\section{Qualitative Examples and Figures}
\label{app:visuals}
\makeatletter\setlength{\@fptop}{0pt}\makeatother

The following examples illustrate the geometry-to-field mapping, stochastic labels, and path extraction. The sampling budget for each field and walk figure is indicated in its panel.

\subsection{Geometry, Fields, and Multiple Starts}

Figure~\ref{fig:field-demo} illustrates the reuse of a goal-conditioned field across multiple initial states: the field represents a map and goal, while path extraction starts from each selected location.
\begin{figure}[H]
\centering
\includegraphics[width=.93\linewidth]{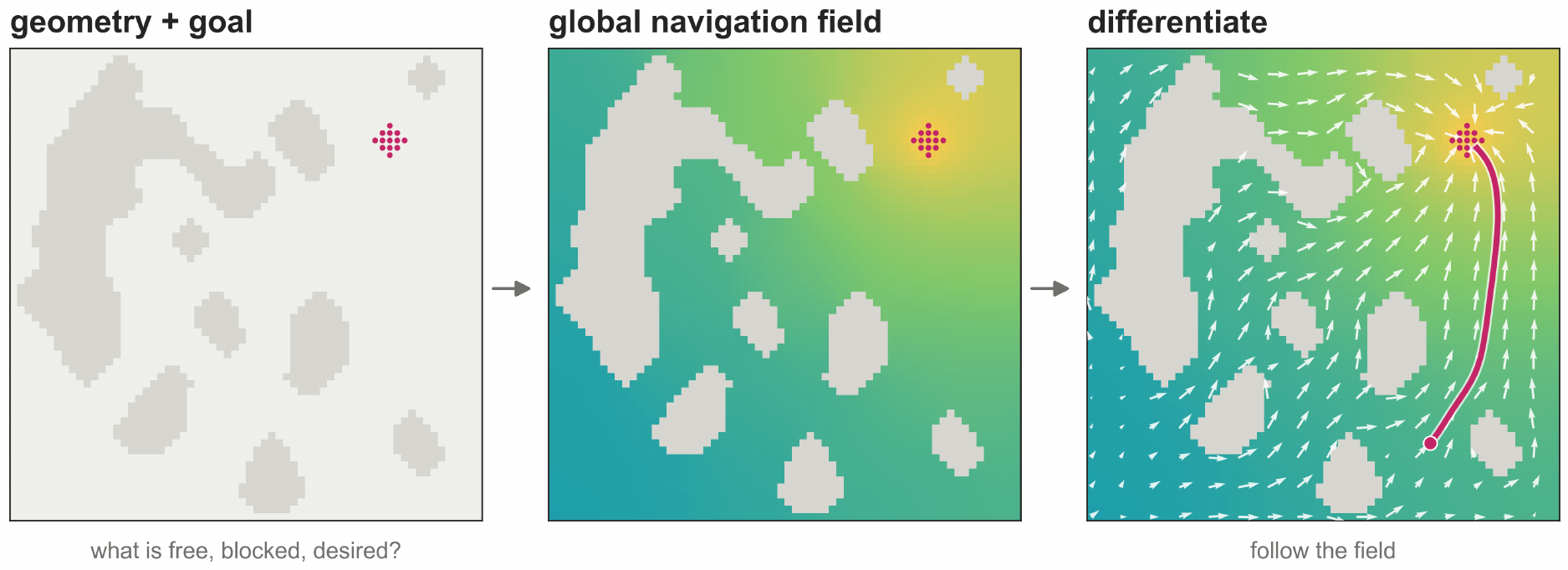}\par\smallskip
\includegraphics[width=.93\linewidth]{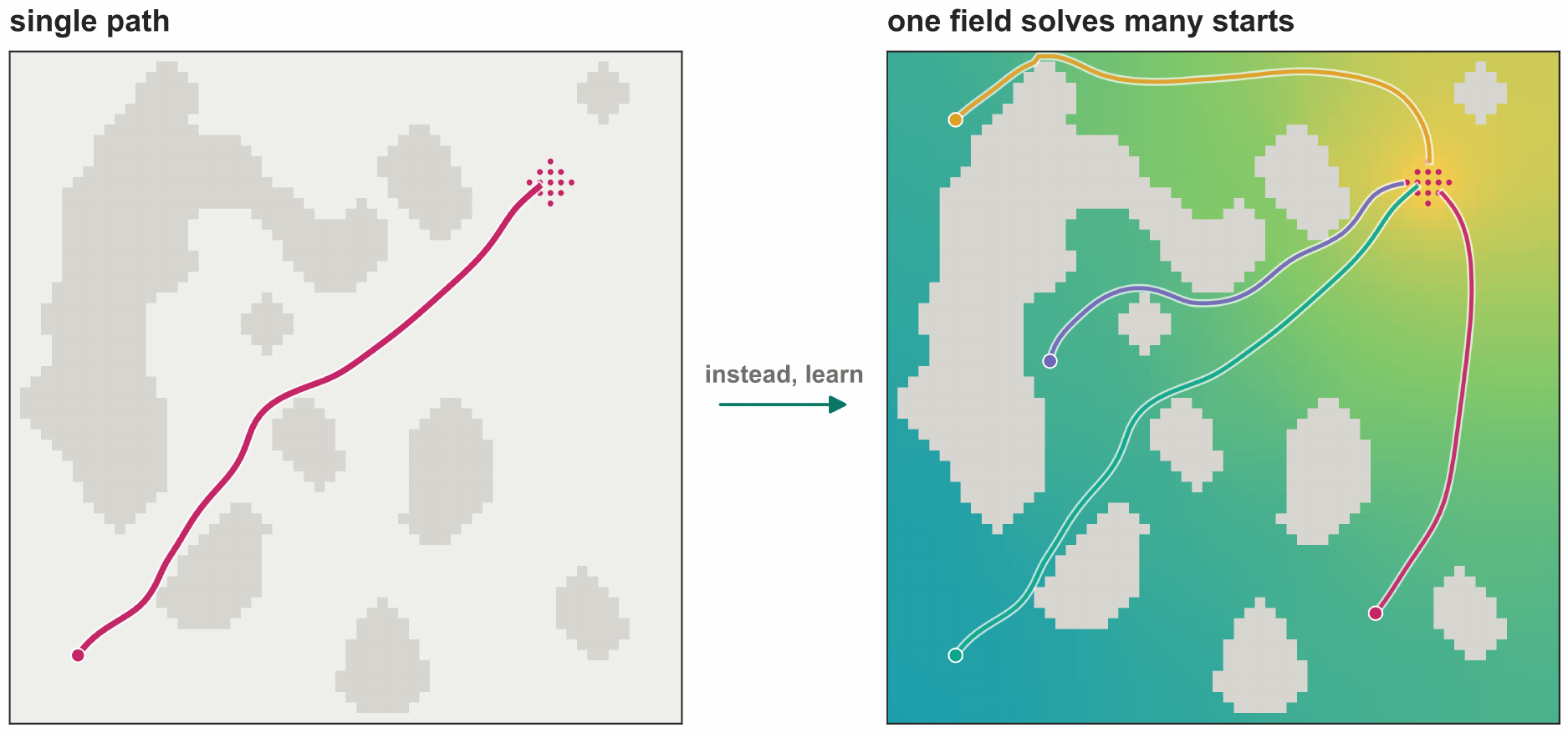}\par\smallskip
\caption{Goal-conditioned field prediction and reuse across initial states.}
\label{fig:field-demo}
\end{figure}
\newpage
\subsection{Walk Geometry and Label Variability}
\label{app:visual-labels}

Figure~\ref{fig:walk-variability} brings together walk geometry, label variability, and spatial sampling cost.

\begin{figure}[!htbp]
\centering
\begin{minipage}[t]{0.48\linewidth}
\centering
\textbf{(a) Walk geometry and boundary treatment}\par
\includegraphics[width=\linewidth]{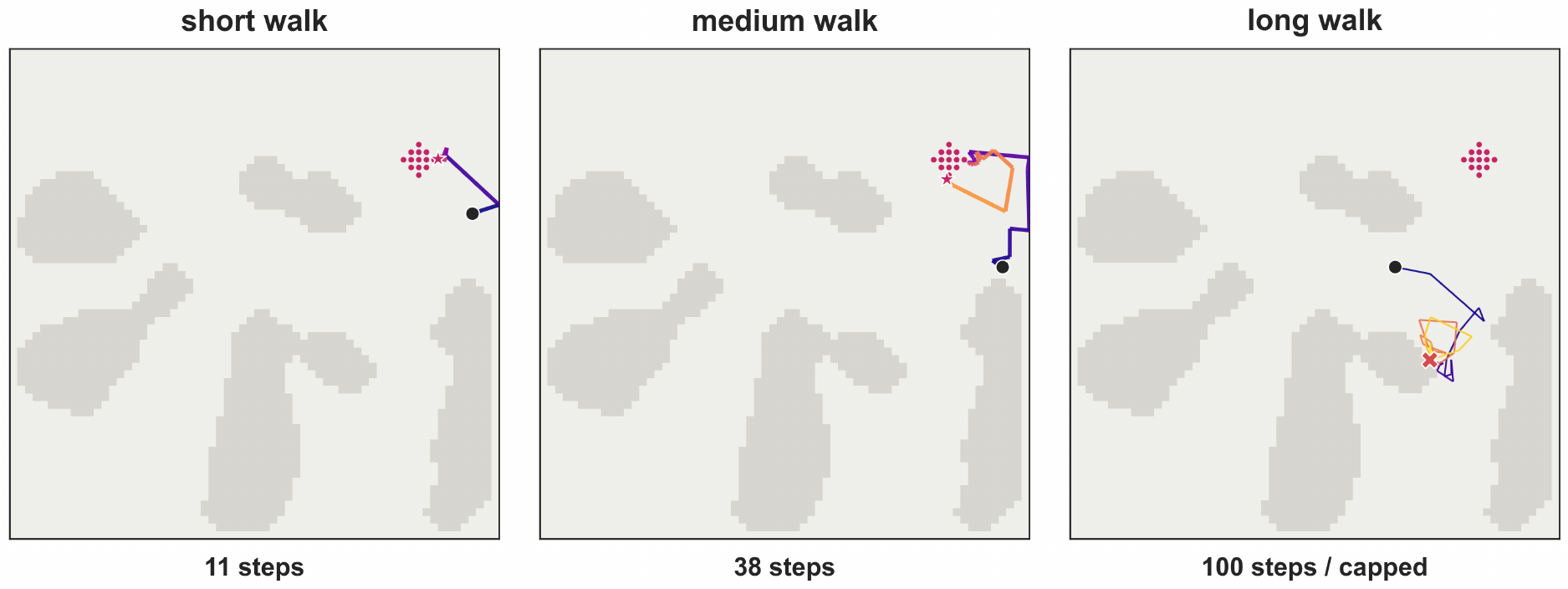}\par
\includegraphics[width=\linewidth]{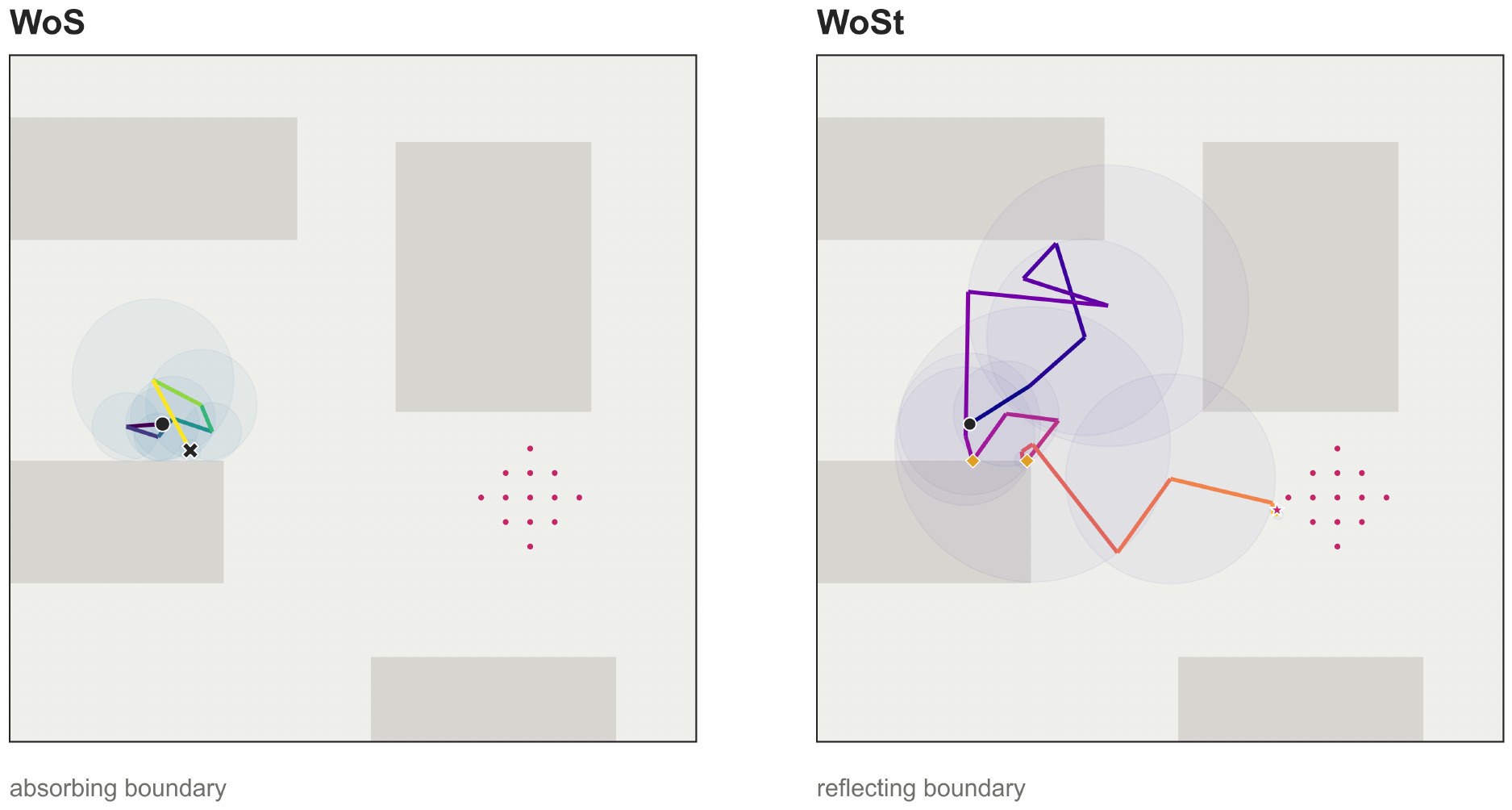}\par
\textbf{(c) Mean walk length}\par
\includegraphics[width=\linewidth]{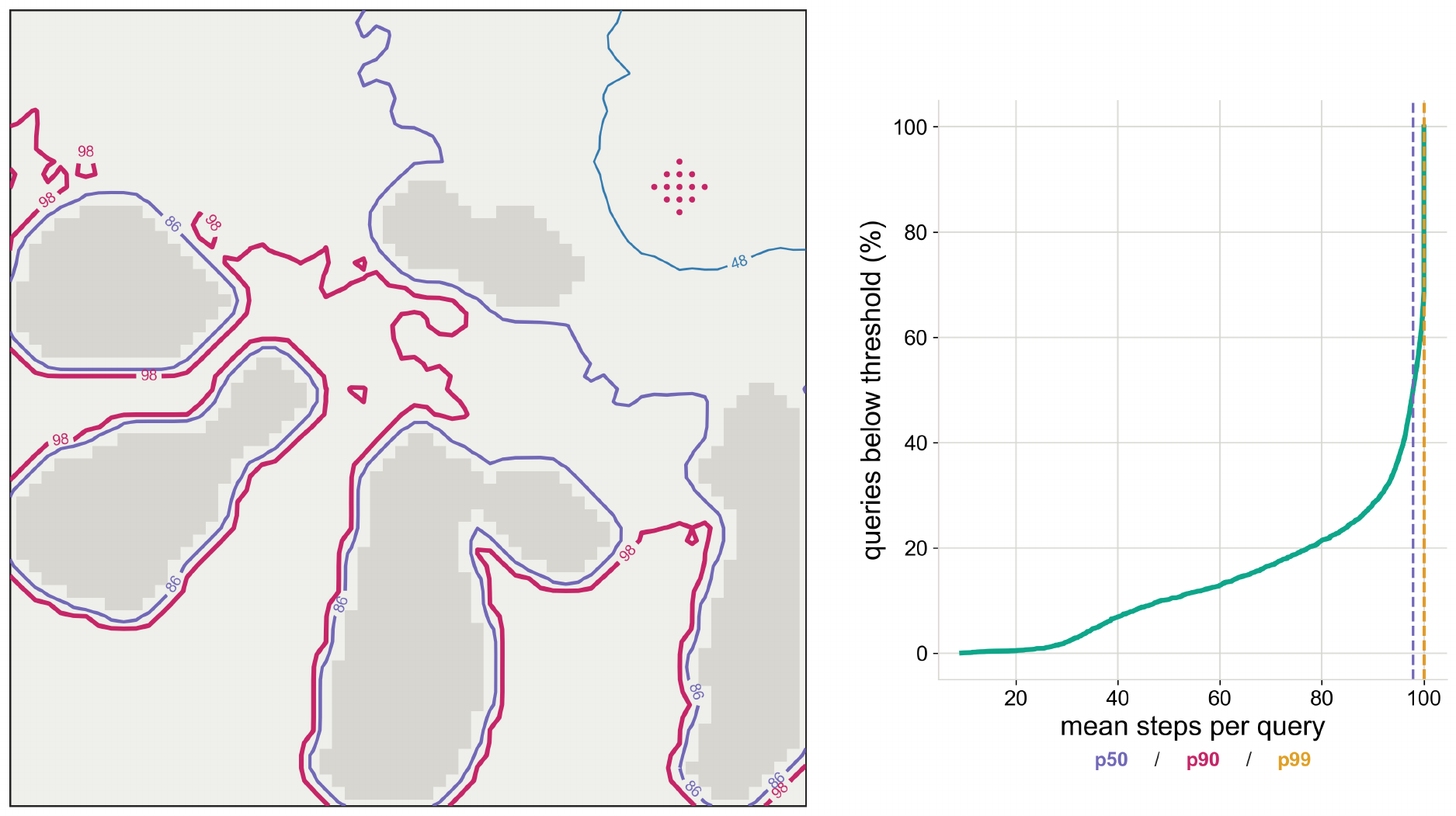}
\end{minipage}\hfill
\begin{minipage}[t]{0.48\linewidth}
\centering
\textbf{(b) Label and gradient variability}\par
\includegraphics[width=\linewidth]{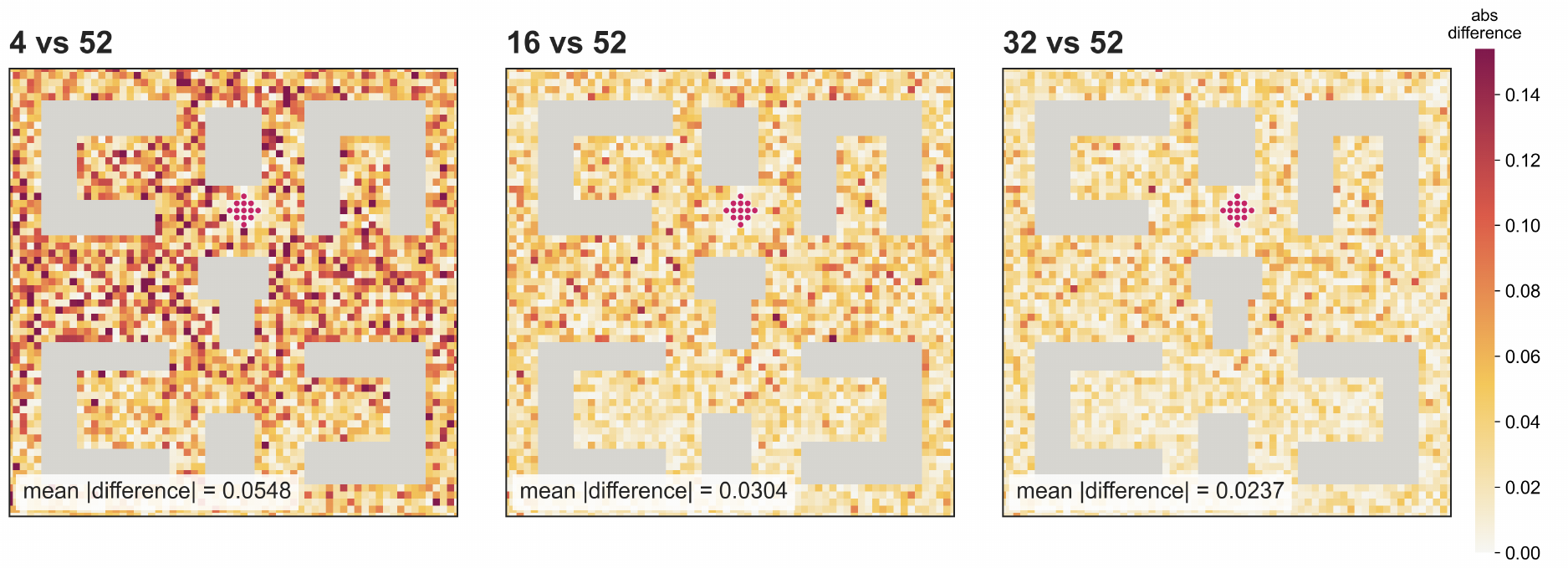}\par
\includegraphics[width=\linewidth]{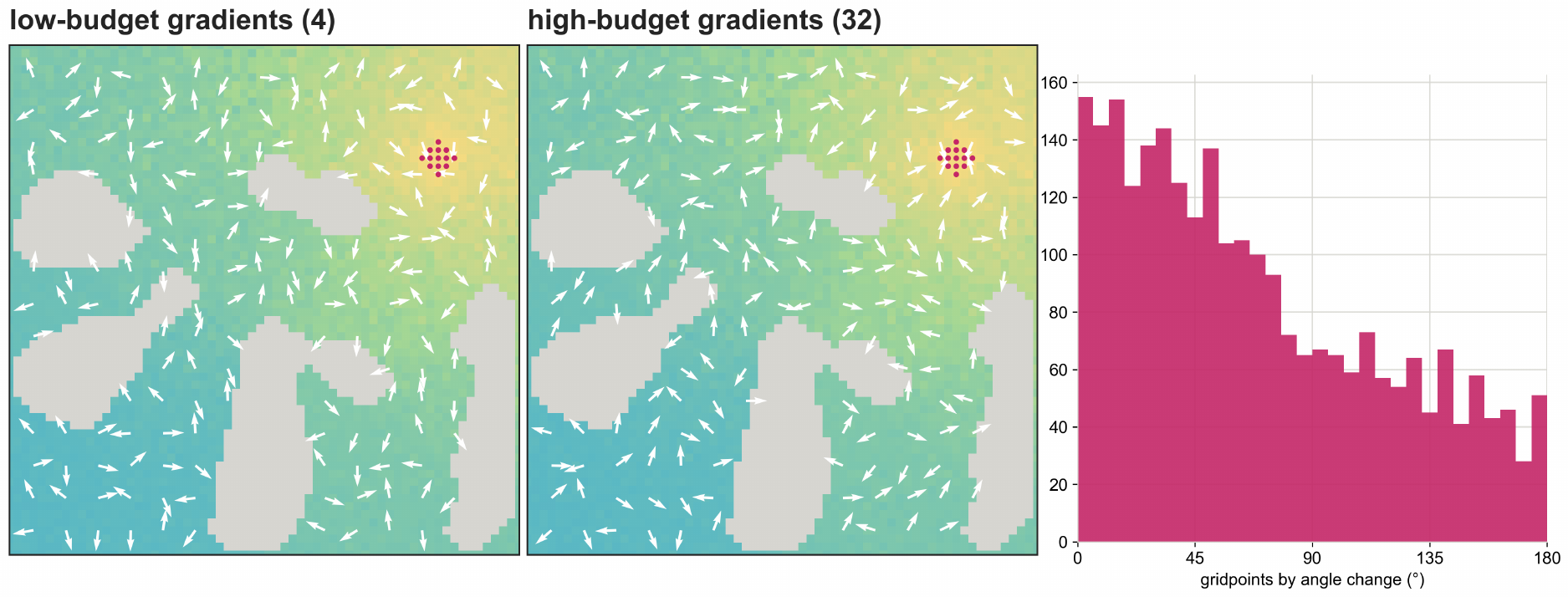}
\end{minipage}
\caption{WoSt walk and label statistics: (a) walk-length variability and boundary treatment; (b) label and gradient variability on a fixed map, with differences measured against a finite Monte Carlo reference; (c) spatial variation in mean walk length.}
\label{fig:walk-variability}
\end{figure}

\FloatBarrier
\subsection{Local Search Example}
\label{app:local-search-example}

\begin{figure}[!htb]
\centering
\includegraphics[width=\linewidth]{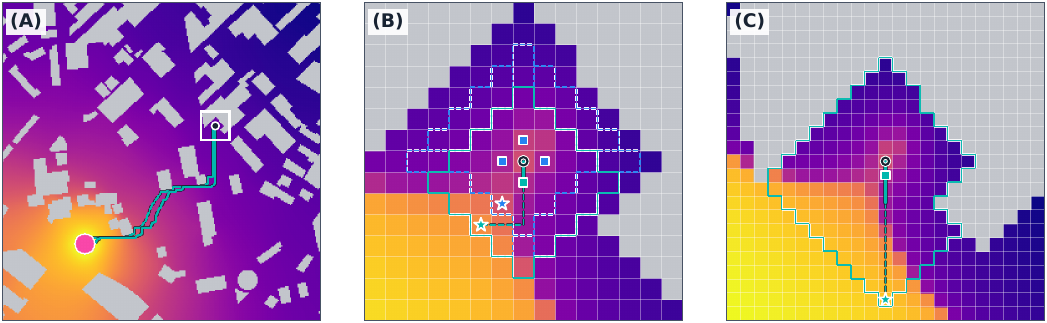}\par
\RaggedRight
\caption{Local search on one fixed urban-map field: (A) full map and routes; (B,C) enlarged neighborhoods for local search (4) and (10). Greater depth exposes endpoints farther along the passage. Gray cells are obstacles; overlays mark candidate regions and routes. Table~\ref{tab:main-planning} reports aggregate controller success.}
\label{fig:local-search-cells}
\end{figure}

\FloatBarrier

\FloatBarrier

\FloatBarrier

\end{document}